\documentclass[11pt]{article}
\pdfoutput=1  % arXiv: compile with pdflatex (all figures are PDF)

\usepackage[utf8]{inputenc}
\usepackage[T1]{fontenc}
\usepackage{lmodern}
\usepackage{microtype}
\usepackage{amsmath,amssymb,amsthm}
\usepackage{booktabs}
\usepackage{graphicx}
\usepackage{algorithm}
\usepackage{algpseudocode}
\usepackage[round]{natbib}
\usepackage[margin=1.05in]{geometry}
\usepackage[colorlinks=true,linkcolor=blue,citecolor=blue,urlcolor=blue]{hyperref}
\usepackage{url}
\usepackage{tikz}
\usetikzlibrary{positioning,arrows.meta,fit,calc}

\newcommand{\score}{\mathrm{GEO}}
\newcommand{\pawc}{\mathrm{PAWC}}
\newtheorem{definition}{Definition}
\newtheorem{assumption}{Assumption}
\newtheorem{problem}{Problem}
\newtheorem{proposition}{Proposition}
\newtheorem{lemma}{Lemma}
\newtheorem{remark}{Remark}

\tikzset{
  papbox/.style={draw, rounded corners=1pt, align=center, font=\footnotesize,
                 inner sep=4pt, minimum height=7mm},
  paparrow/.style={-{Latex[length=2mm]}, thick},
  paplabel/.style={font=\scriptsize\itshape, align=center},
}

\title{Scoring Without the Engine:\\
Validating a Deterministic, Manipulation-Resistant Content Score\\
for Generative Engines, End to End\thanks{All reported numbers reproduce
offline from released artifacts; see Section~\ref{sec:repro}.}}

\author{Elisha Bajemon\\
TW3 Partners (Citead)\\
\texttt{e.bajemon@tw3partners.com}
\and
André-Louis Rochet\\
TW3 Partners (Citead)\\
\texttt{arochet@tw3partners.com}}

\date{July 2026}

\begin{document}
\maketitle

\begin{abstract}
How do you validate a cheap, deterministic proxy for an oracle that is
expensive, rate-limited, and non-stationary? We present a protocol
whose core is a set of \emph{adversarial falsification gates} that
define and select the proxy (negative control, dose response, bounded
amplification, duplication penalty, length neutrality), selected on a
training split and confirmed held-out; around the gates, the protocol
bounds what the proxy can never resolve with a query-conditioned
skyline model, checks the aggregation against external causal evidence
where any exists, and re-measures that evidence on the current oracle
rather than assuming it. We demonstrate the protocol end to end on
Generative Engine Optimization, where the proxy is a deterministic,
auditable content score, and one step fails on the demonstration
domain in exactly the way the protocol is built to detect: re-measuring
the only published causal anchors (2023 effect sizes) on ten modern
engine families (six open-weights, four commercial, including three
proprietary gpt-5.x arms) shows their levers move citation on none of
them, so the anchors are an expired external check, the released
modern anchor vector is near zero, and recalibrating to it strips the
score of its lever-responsive components: what survives, and what the
paper evaluates, is the gate-enforced response surface. The gates buy
a measured property: on a 500-source benchmark of volume-controlled
adversarial edits, amplifying the score's own calibrated levers gains
an attacker at most $6$ points, and this gain decreases with dose;
single-lever dose amplification is provably bounded
(Lemma~\ref{lem:concave}), while the absolute $6$-point cap and the
sub-additivity across levers are empirical findings consistent with
that bound rather than corollaries of it; on detection, standard
web-spam baselines dominate and
out-of-distribution attacks evade the score, so the deployable filter
layers the score over those baselines. A query-conditioned skyline
bounds the score's citation signal (within-query Spearman $0.11$; no
content-only model from our feature family beats it, consistent with,
though not an estimate of, the formal ceiling), repositioning every
query-agnostic score as a quality filter rather than a citation
predictor. The audit trail is part of the method: a query-leakage bug
in our own first ranking evaluation and a failed confidence flag are
disclosed and corrected, and every number reproduces offline from
released artifacts at zero marginal API cost.
\end{abstract}

\section{Introduction}
\label{sec:intro}

Generative engines (AI overviews, conversational search, answer engines)
increasingly mediate access to web content, and a literature has emerged on
\emph{Generative Engine Optimization} (GEO): which properties of a page cause
it to be cited in generated answers \citep{aggarwal2024geo, pinterest2026geo}.
The standard way to evaluate GEO is \emph{answer-side}: query a live engine,
parse its answer, and attribute visibility to sources
\citep{aggarwal2024geo, mageo2026, agenticgeo2026}. Answer-side measurement is
causally the gold standard, but it is poorly suited to benchmarking: it is
paid, rate-limited, non-stationary, and hard to reproduce, and its cost grows
with every site evaluated.

This paper contributes a validation protocol for exactly this
situation, a cheap deterministic proxy standing in for an expensive,
non-stationary generative oracle, and demonstrates it end to end on
GEO: define and select by adversarial falsification gates; bound with a
conditioned skyline; check against external causal evidence where any
exists; and re-measure that evidence on the current oracle. On the
demonstration domain the causal check comes back expired, a result the
protocol is designed to produce rather than a defect it hides. The
protocol separates two
objects that are often conflated.
The first, the protocol's demonstration artifact, is a
\textbf{query-agnostic
content score evaluated as a manipulation-resistant quality filter}
(Sections~\ref{sec:algo}
to~\ref{sec:gaming}): a deterministic score computed from page text
alone, suitable for corpus filtering, editorial quality control, and
public leaderboards, auditable line by line, and costing nothing to
recompute. Such a score cannot see the query, so it should
not be expected to resolve which source an engine will cite for a specific
question; its claims must be established differently, by the
falsification gates and by alignment with historical causal evidence
(the only published effect sizes, which Section~\ref{sec:outcome} shows
are now expired on current engines), and
Section~\ref{sec:gaming} evaluates the latter directly on a
gaming-detection benchmark. The second object
is a \textbf{query-conditioned citation predictor}
(Section~\ref{sec:ltr}): not a competing product but a measuring instrument,
built to quantify the ceiling that query-blindness imposes on the first
object and how much of the remaining signal query--source relevance
recovers. The two artifacts share the same deterministic feature stack but
answer different questions, and we evaluate each against the standard
appropriate to it (Figure~\ref{fig:twoobjects}).

Four results organize the paper. First, the gates buy a property no
detection baseline provides: an attacker amplifying the score's
calibrated levers is capped at $6$ points, with gains that shrink as
the attack intensifies and do not compound across levers; the hardened
benchmark also establishes the division of labor, spam-style
degradation detection belongs to standard web-spam baselines, bounded
inflatability belongs to the score. Second, the protocol's re-measurement step returns a negative that
retires its own first step: under paired, volume-controlled edits on
ten engine families, none of the three strongest 2023 GEO levers moves
citation, and the self-measured modern anchor vector we release is
statistically indistinguishable from zero where it is not nominally
negative, the first anchor measurement on current engines; the causal
anchoring is thereby demoted to an expired external check, and the
gates carry the definition of the score. Third, the formal bound and its empirical counterpart are kept
distinct: Proposition~\ref{prop:ceiling} shows some ceiling
$C^{\max}$ binds every query-blind scorer, and, \emph{relative to our
content-feature family}, a flexible ranker fails to beat the fixed
score on unseen queries ($0.10$ versus $0.12$), evidence consistent
with the bound though not an estimate of $C^{\max}$ itself.
% TO-PRODUCE (authors): estimate C_max directly from the released
% five-source outcome records (e.g., the best achievable within-query
% ordering given only source identities), so the empirical gap to the
% formal ceiling is quantified rather than argued.
Fourth,
query conditioning roughly triples the measured signal, and the
protocol's audit discipline is what makes that number trustworthy: our
own first evaluation overstated it through query leakage across engine
arms, which a query-blind probe caught and query-disjoint folds
corrected; the strongest single conditioned signal is an open-weights
cross-encoder, and a regex-only ranker recovers under two fifths of the
measured gap.

Beyond GEO, the paper is a case study in validating a cheap
deterministic proxy for an expensive, non-stationary generative oracle:
define by adversarial constraints, check against the strongest causal
evidence available, report per-parameter identifiability rather than
optimality, bound the proxy with a conditioned skyline, and re-measure
the anchors rather than assuming them. The same problem arises from
LLM-judge replacement \citep{zheng2023judge} to corpus filtering
\citep{wenzek2020ccnet}; GEO supplies published causal anchors and a
measurable engine outcome, the rare ingredients that make the protocol
demonstrable end to end.

\begin{figure}[t]
\centering
\begin{tikzpicture}[node distance=4mm and 8mm]
\node[papbox, fill=blue!6, minimum width=34mm] (feat)
  {deterministic feature stack\\ (regex + statistics, no engine)};
\node[papbox, above left=9mm and -6mm of feat, minimum width=46mm] (score)
  {\textbf{query-agnostic GEO score}\\ $s(x)$, one page at a time};
\node[papbox, above right=9mm and -6mm of feat, minimum width=46mm] (ltr)
  {\textbf{query-conditioned predictor}\\ rank $x_{1:K}$ given query $q$};
\node[paplabel, above=1.5mm of score, align=center] (sval)
  {validated by: falsification gates $+$ alignment with\\ (now-expired) causal anchors\\
   + adversarial gates (Sections~\ref{sec:method}--\ref{sec:ab})};
\node[paplabel, above=1.5mm of ltr, align=center] (lval)
  {validated by: held-out citation outcomes,\\
   leave-one-engine-out (Section~\ref{sec:ltr})};
\draw[paparrow] (feat) -- (score);
\draw[paparrow] (feat) -- (ltr);
\draw[paparrow, dashed] (ltr.west) -- node[paplabel, above, sloped]
  {+ query--source relevance} (score.east);
\end{tikzpicture}
\caption{The two artifacts and the claim each one is allowed to make. Both
consume the same deterministic features; only the predictor sees the query,
and neither requires engine access at inference.}
\label{fig:twoobjects}
\end{figure}
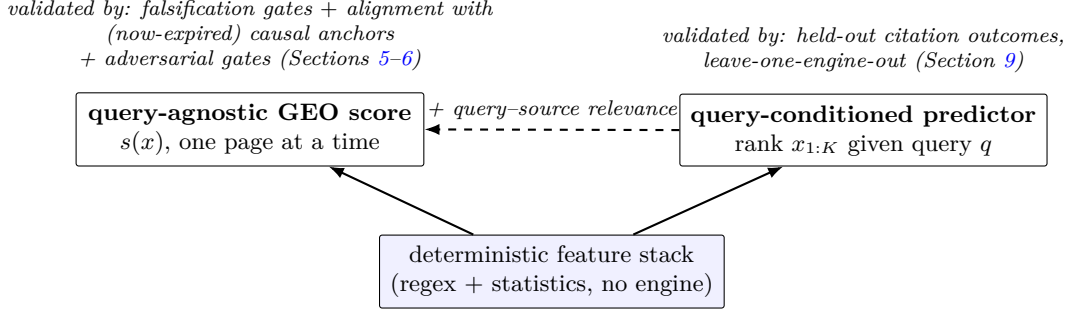

Our contribution is the protocol and its three demonstrated outputs on GEO:

\begin{enumerate}
\item \textbf{A manipulation-resistant, deterministic content score,
  evaluated as such} (Sections~\ref{sec:algo} to~\ref{sec:gaming}):
  eleven regex and statistical sub-components, a concave anti-stuffing
  transform whose single-lever \emph{dose} amplification is provably
  bounded (Lemma~\ref{lem:concave}; the absolute gain cap and the
  cross-lever sub-additivity are the measured, not proved,
  counterparts), weights aligned with the published causal
  effect sizes \citep{aggarwal2024geo} under volume-controlled edits and
  five falsification gates computed on the training split and confirmed
  held-out, with the per-weight identifiability analysis
  (Section~\ref{sec:ident}) making the reach of that alignment precise:
  two weights are pinned by the anchors, two forced by the gates, and
  the remainder are inherited from the prior observational weighting
  and labeled as such, so "calibrated" describes a subset of the weight
  vector, not all of it. A hardened gaming benchmark
  (Section~\ref{sec:gaming}) evaluates the score in the web-spam lineage
  it belongs to, against out-of-distribution attacks and standard
  no-fitting baselines: the headline property is the bounded attacker
  (at most $+6$ points from any lever, dose-decreasing, sub-additive);
  detection is narrower, dominated by trivial baselines where it works
  at all, and absent on attacks outside the calibration families, all
  of which the paper states rather than buries. Scoring runs
  offline on commodity hardware (about 70~ms for typical pages).
\item \textbf{Self-measured causal anchors on ten modern engine
  families, and what they imply for calibration}
  (Section~\ref{sec:outcome}): the three strongest 2023-era GEO
  interventions (quotation, statistics, cite-sources) move citation on
  \emph{none} of ten engine families (six open-weights, four commercial,
  including a July-2026 replication on three proprietary gpt-5.x arms)
  under a paired, volume-controlled test. We turn that negative into the
  anchor measurement the field lacks: the released modern anchor vector
  is near zero (one lever nominally negative), a recalibration against
  it degenerates exactly as it should (the released
  modern-anchor recalibration zeroes the lever-responsive features), and the
  conclusion is a repositioning of calibration itself: what it buys on
  current engines is the gate-enforced response surface, not anchor
  recency. Observationally the score still predicts citation modestly
  but reliably (pooled within-query Spearman $0.114$ on the seven
  original arms, $p < 10^{-8}$;
  replicated out of family at $0.118$ on the three gpt-5.x arms).
\item \textbf{A query-conditioned bound on content-only scoring, and the
  audit trail that makes it credible}
  (Section~\ref{sec:ltr}): a formal ceiling argument
  (Proposition~\ref{prop:ceiling}) plus its empirical counterpart, a
  flexible ranker restricted to content features that fails to beat the
  fixed score on unseen queries ($\rho = 0.10$ versus $0.12$), while
  query-conditioned models reach $\rho \approx 0.4$ and transfer to
  held-out commercial engine families. The same study contributes a
  cautionary audit: fold construction that shares queries across engine
  arms inflates every fitted-ranker number (e.g.\ $0.57$ versus $0.36$
  for the full model, both on the audited seven-arm dataset), a leakage we caught with a query-blind probe,
  disclose, and correct; a dispersion-based confidence flag that failed
  held-out confirmation is reported in the same tradition.
\end{enumerate}

Throughout, the audit discipline is part of the method rather than an
afterthought: the same intervention harness exposed a feature-coverage
bug, a bug in our own negative-control edit
(Sections~\ref{sec:frfix} and~\ref{sec:ksbug}), and the ranking-evaluation
leakage above, all fixed and disclosed, with every capture regenerated.

\section{Related work}
\label{sec:related}

The paper sits in the evaluation-and-measurement lineage before the GEO
lineage: its research object is how to validate a deterministic proxy
for a generative oracle, the same problem faced by LLM-as-judge
replacement \citep{zheng2023judge}, heuristic corpus filtering
\citep{wenzek2020ccnet}, and construct-validity methodology
\citep{jacobs2021measurement}, under the Goodhart constraint that
optimizing a proxy degrades it \citep{manheim2018goodhart}. GEO is the
demonstration domain, chosen because it uniquely supplies published
causal effect sizes and a measurable engine outcome.

\paragraph{Generative engine optimization.}{\sloppy{} \citet{aggarwal2024geo}
introduced GEO and GEO-Bench~\citep{geobench} and established, via paired
interventions on 10k queries, the relative Position-Adjusted Word Count
(PAWC) gains of nine content edits; we use their effect sizes as causal
anchors. \citet{khosravi2026aio} give causal evidence that AI search
summaries divert traffic from source sites, the economic stake of the
problem. \citet{pfrommer2024ranking} study adversarial ranking manipulation
of conversational engines, motivating our adversarial gates.
\citet{puerto2025cseo} report that many conversational-SEO tactics are
ineffective or zero-sum under competition, a threat to validity we inherit
and discuss in Section~\ref{sec:threats}. \citet{mageo2026} and
\citet{agenticgeo2026} measure visibility answer-side with live engines;
\citet{pinterest2026geo} document production GEO at scale. Closest to our
first artifact, \citet{yu2026structural} engineer structural content
features that correlate with citation; they validate observationally
against engine outputs, whereas our score is calibrated against published
causal effect sizes, gated adversarially, and its weights are given an
identifiability analysis rather than a correlational justification. The difference is the validation object, not the feature
engineering: their score is validated by observational agreement with
engine outputs, exactly the regime Section~\ref{sec:obs} shows to be
structurally overfit at feasible sample sizes, whereas ours is defined
by falsification gates, checked against causal evidence, and shipped
with per-weight identifiability. No head-to-head number exists because
their features and outcome corpus are not released in comparable form;
an empirical comparison on a shared corpus is future work, stated here
rather than implied. On the
engine side, \citet{liu2023verifiability} document how imperfectly
generative search engines support their citations, and \citet{gao2023alce}
benchmark citation generation in LLMs; both study the model behavior whose
content-side correlates we score.\par}

\paragraph{Deterministic text-quality measurement.} Scoring text with
closed-form statistics is an old idea: readability formulas date to
\citet{flesch1948readability}, and \citet{bendersky2011quality} showed that
deterministic document-quality signals (readability, stopword ratios,
structure) improve web ranking, the ad-hoc-retrieval ancestor of a GEO
score. Our adversarial gates likewise inherit from the web-spam
literature, where deterministic content statistics were used to
\emph{detect} gaming \citep{ntoulas2006spam} and ranking systems were
hardened against it \citep{gyongyi2004trustrank}; the negative-control
gate is that tradition turned into a falsification test, and the gates as
a family are a guard against the proxy-gaming failure modes catalogued in
the Goodhart's-law literature \citep{manheim2018goodhart}. Our entropy
features inherit from information-theoretic accounts of
text: \citet{shannon1948} for the quantity itself,
\citet{genzel2002entropy} for entropy-rate regularities in well-formed
text, and the Uniform Information Density literature
\citep{jaeger2010redundancy, meister2021uid} for the hypothesis that
effective text distributes information evenly, the intuition behind
information-density features. Corpus-filtering pipelines likewise use
deterministic quality proxies at web scale \citep{wenzek2020ccnet}. For the
methodology of arguing that a computed score measures a construct, we
follow the construct-validity tradition of \citet{cronbach1955construct}
in its measurement-modeling formulation for computational systems
\citep{jacobs2021measurement}; the same stand-in-for-an-expensive-oracle
problem arises when LLM judges replace human evaluation
\citep{zheng2023judge}. Our ranker is LambdaMART \citep{wu2010lambdamart}
as implemented in LightGBM. To our knowledge, no prior work calibrates or
selects a deterministic content-side score against published causal
effect sizes, nor gates it adversarially, nor reports per-weight
identifiability for such a score.

\section{Problem statement and assumptions}
\label{sec:problem}

Let $\mathcal{X}$ be the space of page texts (markdown) and let $E$ denote a
generative engine. For a query $q$ with candidate sources
$x_1, \dots, x_K \in \mathcal{X}$, the engine produces an answer whose
visibility allocation $V_E(x_k \mid q, x_{1:K}) \in [0,1]$ (a PAWC-like
share, Section~\ref{sec:outcome}) is the causal quantity of interest.
Answer-side evaluation measures $V_E$ directly but is paid, rate-limited,
and non-stationary. The engine-free problem is to construct content-side
functions that stand in for parts of it, with claims no stronger than what
each function can support. We formalize the two objects of this paper.

\begin{problem}[Query-agnostic score]
\label{prob:score}
Find a deterministic $s : \mathcal{X} \to [0, 100]$, computable offline with
no engine access, such that (i) for every content intervention $m$ with
published causal visibility gain $\gamma_m$, the volume-controlled response
$\Delta_m(s)$ of Eq.~\eqref{eq:delta} is positively aligned with
$(\gamma_m)_m$; (ii) $s$ is robust to adversarial amplification of any
single lever (the gates of Section~\ref{sec:gates}); and (iii) $s$ is
auditable: every sub-component is a closed-form function of the text.
$s$ is \emph{not} required to predict which source an engine cites for a
given query; that is Problem~\ref{prob:ltr}.
\end{problem}

\begin{problem}[Query-conditioned citation prediction]
\label{prob:ltr}
Given $(q, x_{1:K})$ and an engine family $E$, predict the within-query
ranking of $V_E(x_k \mid q, x_{1:K})$, using only deterministic features of
$q$ and $x_{1:K}$, evaluated on held-out queries and held-out engine
families.
\end{problem}

The separation between the two problems is not a modeling preference but a
theorem: query-blindness imposes a ceiling that no feature engineering can
lift.

\begin{proposition}[Query-blindness ceiling]
\label{prop:ceiling}
For an unordered pair of sources $\{x, x'\}$, let
$p(x, x') = \Pr\big(V_E(x \mid q, \cdot) > V_E(x' \mid q, \cdot) \,\big|\,
x, x'\big)$ be the probability, over the queries in which the pair
competes, that $x$ is the more visible source. Any query-agnostic scorer
$s : \mathcal{X} \to \mathbb{R}$ orders each pair identically for every
query, so its expected pairwise concordance satisfies
\begin{equation}
C(s) \;\le\; \mathbb{E}_{\{x, x'\}}\big[\max\{p(x, x'),\, 1 - p(x, x')\}\big]
\;=:\; C^{\max},
\end{equation}
and $C^{\max} < 1$ whenever visibility genuinely depends on the query,
i.e.\ $0 < p(x, x') < 1$ on a set of pairs of positive measure. Only
conditioning on $q$ can close the gap $1 - C^{\max}$.
\end{proposition}

\begin{proof}
A query-agnostic $s$ fixes one decision per unordered pair, which is
correct with probability $p(x, x')$ or $1 - p(x, x')$; the per-pair
contribution to $C(s)$ is therefore at most the larger of the two.
Transitivity of the total order induced by $s$ can only make the
attainable value smaller. If visibility does not depend on the query,
$p \in \{0, 1\}$ almost surely and the bound is vacuous.
\end{proof}

The bound explains why Problem~\ref{prob:score} must not be judged by
Problem~\ref{prob:ltr}'s standard; Section~\ref{sec:ltr} probes whether
the fixed score is already near the best a content-only model can do (a
flexible ranker restricted to content features fails to beat it), which
lower-bounds the score's efficiency within its function class; it does
not estimate $C^{\max}$ itself, which we leave open.

Problem~\ref{prob:score} is an \emph{estimation problem with an
identification caveat}: the anchor set has $|\mathcal{M}| = 5$ interventions
while the score has $d = 11$ weights, so weight vectors are underdetermined
by alignment alone. We therefore never claim a unique optimal weighting;
Section~\ref{sec:ident} characterizes the \emph{admissible set} of weights
instead, and the selection among aggregation families is made by the
falsification gates, not by the (saturated) alignment metric. The claims
rest on the following explicit assumptions.

\begin{assumption}[Anchor validity]
\label{as:anchor}
The published PAWC effect sizes of \citet{aggarwal2024geo} reflect the true
relative causal ordering of the five deterministic interventions on their
engines and corpus.
\end{assumption}

\begin{assumption}[Anchor transport]
\label{as:transport}
The relative ordering of Assumption~\ref{as:anchor} carries over to
GEO-Bench sources under our deterministic edit realizations. This is an
assumption about the \emph{calibration target}, not about current engines:
Section~\ref{sec:outcome} tests the levers against ten modern engine
families and finds they do not transfer, a result we report rather than
assume away. The alignment claim is therefore relative to historical
causal evidence that no longer describes current engines, and
Sections~\ref{sec:method} and~\ref{sec:threats} state what survives
that expiry (the gate-enforced response surface).
\end{assumption}

\begin{assumption}[Volume control]
\label{as:volume}
Appending a content-neutral filler block of the same added length isolates
the content effect of an edit from its length effect
(Eq.~\eqref{eq:delta}).
\end{assumption}

\begin{assumption}[Determinism]
\label{as:determinism}
Under a pinned environment (fixed library versions and seeds), every
sub-component $f_i$ and every reported statistic is a deterministic function
of the input text and the released artifacts.
\end{assumption}

Under this framing, calibration is a \emph{constrained} problem: alignment
is the objective, and the gates are feasibility constraints,
\begin{equation}
w^\star \;\in\; \operatorname*{arg\,max}_{w \in \Delta^{d-1}}\;
\mathrm{align}\big(\Delta(s_w),\, \gamma\big)
\qquad \text{s.t.}\qquad
\mathrm{gate}_j(s_w) \text{ passes},\; j = 1, \dots, 5,
\label{eq:constrained}
\end{equation}
solved in practice in two stages (Section~\ref{sec:ab}): each aggregation
family is fitted by the regularized surrogate of Eq.~\eqref{eq:fit}, and the
gates then eliminate infeasible families. The two-stage form is what makes
the trade-off visible: the quantile family is the best on one secondary
objective (French transfer) yet violates the dose-response constraint and
is rejected; a scalarized single objective would have hidden exactly this.

\section{The scoring algorithm}
\label{sec:algo}

\subsection{Sub-components}

Let $x$ be the page text (markdown). The score uses $d = 11$ deterministic
sub-components $f_i(x) \in [0, 100]$, each computable with regex, counting,
\texttt{scipy}/\texttt{numpy}, \texttt{wordfreq}, and \texttt{scikit-learn}
TF-IDF (bag-of-words linear algebra; deterministic for a pinned environment).
Table~\ref{tab:features} lists them; Appendix~\ref{app:features} gives the
exact formula, normalization, degenerate values, and a worked numeric
example for every sub-component, all verified against the released
implementation. No LLM, no learned embeddings, no network access is
involved at scoring time.

Three properties of the feature set matter for interpreting everything that
follows, and we state them here rather than in a footnote. First,
\texttt{information\_density} is a composite whose dominant term ($70\%$) is
exactly \texttt{shannon\_entropy}, so the two features double-count one
signal (joint pre-transform weight $\approx 0.42$; Pearson $0.94$ on
GEO-Bench) and the entropy family is best read as a single construct split
across two artifact names; a merge refit confirms the split is immaterial
(Section~\ref{sec:ident}). Second, two features keep historical artifact
names that no longer describe their formulas: \texttt{citation\_f1} is a
weighted arithmetic mean rather than a harmonic F1, and \texttt{ndcg\_score}
uses the DCG ratio only as stored metadata (Appendices~\ref{app:citation}
and~\ref{app:ndcg}). Third, on corpora of plain-text passages without
markdown structure, the three block-structural features (\texttt{mmr},
\texttt{ndcg}, \texttt{semantic\_redundancy}) sit at their degenerate
constant for most sources (Appendix~\ref{app:diagnostics}); this is the
mechanism behind the parse-failure proxy excluded from the ranking study of
Section~\ref{sec:ltr}.

\subsection{Aggregation}

\begin{definition}[Content score]
With $g(u) = 100\sqrt{u/100}$ applied element-wise and weights
$w \in \Delta^{d-1}$ (non-negative, summing to one),
\begin{equation}
\score_{\mathrm{content}}(x) \;=\; \sum_{i=1}^{d} w_i \, g\!\big(f_i(x)\big).
\label{eq:content}
\end{equation}
\end{definition}

The concave transform $g$ is an anti-stuffing device: the marginal
return of pushing any single sub-component toward its ceiling decays as
$g'(u) \propto u^{-1/2}$, so an attacker optimizing one feature (quote
stuffing, statistic stuffing) saturates, while balanced improvements are
rewarded. The protection is not only empirical; for the attacker's
canonical starting point it is a theorem.

\begin{lemma}[Concave transforms bound single-lever amplification]
\label{lem:concave}
Suppose an edit at dose $k$ moves a single sub-component linearly,
$f_i \mapsto \min(f_i + k\delta, 100)$ with $\delta > 0$, leaving the
others unchanged. For any concave $g$ with $g(0) = 0$, the transformed
dose-$k$ response satisfies $\Delta^{(k)} / \Delta^{(1)} \le k$, with
equality only if $g$ is linear on the traversed range; for
$g(u) = 100\sqrt{u/100}$ and a feature starting at its floor
($f_i = 0$), the ratio is exactly $\sqrt{k}$ before the cap binds.
\end{lemma}

\begin{proof}
With $f_i = 0$ and no cap, $\Delta^{(k)} = w_i\, g(k\delta) \le
k\, w_i\, g(\delta) = k \Delta^{(1)}$ by subadditivity of concave
functions vanishing at zero; for the square root,
$g(k\delta)/g(\delta) = \sqrt{k}$ exactly. The cap at $100$ only lowers
$\Delta^{(k)}$.
\end{proof}

\begin{remark}
The dose-8 saturation gate of Section~\ref{sec:gates} demands a ratio of
at most $3$: from the floor, the square root passes it structurally
($\sqrt{8} \approx 2.83$) where the identity transform attains the worst
case $8$. For a feature far from its floor the square root is locally
linear and the $[0, 100]$ cap does the bounding instead; and empirical
feature responses are not exactly linear in dose. The gate therefore
remains an empirical test; the lemma explains why the sqrt family passes
it with margin (per-lever dose-8/dose-1 ratios $0.55$ for statistics to
$0.96$ for quotation; the \texttt{game} column of Table~\ref{tab:ab}
reports the per-strategy maximum of these, $0.96$ for the selected
family). Because this holds for every concave $g$ with $g(0) = 0$, the bound is a
property of the transform family, not of the threshold: no choice of the
saturation constant can make the sqrt family violate it, or make the
identity family satisfy it from the floor. The gate selects the family;
it does not manufacture the bound.
\end{remark}

Section~\ref{sec:ab} shows this family empirically dominates the
raw-linear production family on dose-response and adversarial gates at equal
anchor alignment.

The page-level score blends content with a freshness component
$F(x) \in [0,100]$ (publication and update recency signals):
\begin{equation}
\score(x) \;=\; 0.92\,\score_{\mathrm{content}}(x) + 0.08\,F(x),
\label{eq:blend}
\end{equation}
followed by word-count caps that bound the score of very short pages (35
points under 100 words, 50 under 200, 65 under 300;
Appendix~\ref{app:aggregation}). Naturalness (AI-detection) and brand trust
are reported separately and are \emph{not} part of the score.

\begin{definition}[Domain ranking]
A domain with $n$ scored pages of mean $\bar{s}$ is ranked by the shrunk
estimate $\tilde{s} = (k p + n \bar{s})/(k + n)$ with prior $p = 50$ and
strength $k = 5$, preventing low-coverage domains from dominating the
leaderboard.
\end{definition}

Figure~\ref{fig:pipeline} summarizes the full scoring path, and
Figure~\ref{fig:taxonomy} groups the eleven sub-components by construct
family with the role each family plays in the calibration
(Section~\ref{sec:ident}).

\begin{figure}[t]
\centering
\begin{tikzpicture}[node distance=3mm and 6mm]
\node[papbox] (x) {page text $x$\\ (markdown)};
\node[papbox, right=of x, minimum width=30mm] (f)
  {11 sub-components\\ $f_i(x) \in [0,100]$\\ (App.~\ref{app:features})};
\node[papbox, right=of f] (g) {$g(u) = 100\sqrt{u/100}$\\ concave transform};
\node[papbox, right=of g] (w) {$\sum_i w_i\, g(f_i)$\\ calibrated weights};
\node[papbox, below=6mm of w, minimum width=30mm] (blend)
  {$0.92 \cdot \mathrm{content} + 0.08 \cdot F(x)$\\ freshness blend};
\node[papbox, left=of blend] (caps) {word-count caps\\ (35/50/65 pts)};
\node[papbox, left=of caps] (geo) {$\score(x)$};
\node[papbox, left=of geo] (dom) {domain ranking\\ (shrinkage, Def.~2)};
\draw[paparrow] (x) -- (f);
\draw[paparrow] (f) -- (g);
\draw[paparrow] (g) -- (w);
\draw[paparrow] (w) -- (blend);
\draw[paparrow] (blend) -- (caps);
\draw[paparrow] (caps) -- (geo);
\draw[paparrow] (geo) -- (dom);
\end{tikzpicture}
\caption{The scoring pipeline. Everything is deterministic and offline:
regex and counting features, a fixed concave transform, calibrated weights,
a freshness blend, and degenerate-input caps. Naturalness and brand trust
are reported outside the score.}
\label{fig:pipeline}
\end{figure}
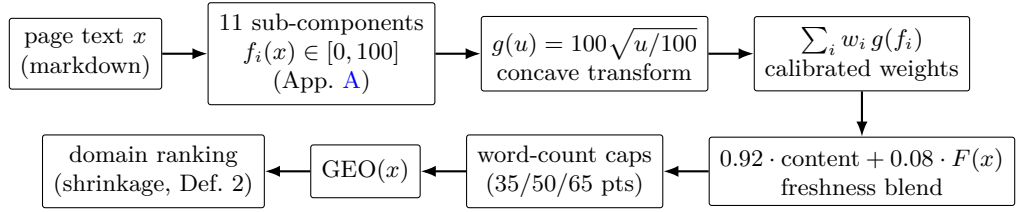

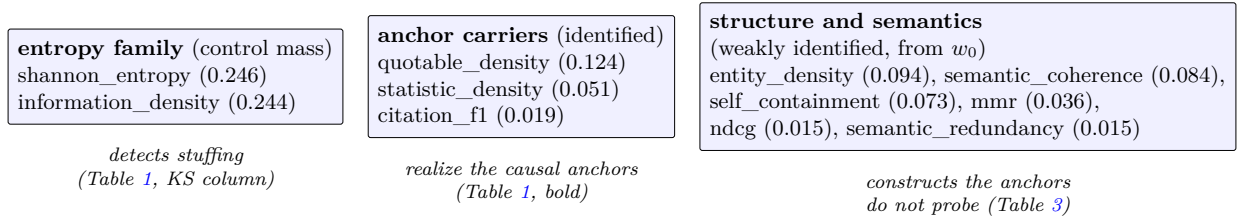
\begin{figure}[t]
\centering
\resizebox{\linewidth}{!}{%
\begin{tikzpicture}[node distance=3mm and 3.5mm]
\node[papbox, fill=blue!6, minimum width=44mm, align=left] (ent)
  {\textbf{entropy family} (control mass)\\
   shannon\_entropy (0.246)\\ information\_density (0.244)};
\node[papbox, fill=blue!6, right=of ent, minimum width=44mm, align=left] (car)
  {\textbf{anchor carriers} (identified)\\
   quotable\_density (0.124)\\ statistic\_density (0.051)\\
   citation\_f1 (0.019)};
\node[papbox, fill=blue!6, right=of car, minimum width=46mm, align=left] (str)
  {\textbf{structure and semantics}\\ (weakly identified, from $w_0$)\\
   entity\_density (0.094), semantic\_coherence (0.084),\\
   self\_containment (0.073), mmr (0.036),\\
   ndcg (0.015), semantic\_redundancy (0.015)};
\node[paplabel, below=2mm of ent] {detects stuffing\\ (Table~\ref{tab:featanchor}, KS column)};
\node[paplabel, below=2mm of car] {realize the causal anchors\\ (Table~\ref{tab:featanchor}, bold)};
\node[paplabel, below=2mm of str] {constructs the anchors\\ do not probe (Table~\ref{tab:ident})};
\end{tikzpicture}}
\caption{The eleven sub-components grouped by the role the identifiability
analysis assigns them. Weights shown are the released values.}
\label{fig:taxonomy}
\end{figure}

\section{Constraints first: gates, then anchors}
\label{sec:method}

A referee should ask, and we answer up front: since
Section~\ref{sec:outcome} will show the published anchor effects are dead
on modern engines, why keep causal-anchor language at all, rather than
specifying the desired response-surface properties directly? Our answer
is that this is, in substance, what the pipeline already does, and this
section is written in that order. The score is \emph{defined} by the
falsification gates: a negative response to stuffing, monotone
non-degenerate dose response, bounded single-lever amplification, a
duplication penalty, and length neutrality, selected on the training
split and confirmed held-out. The published 2023 effect sizes enter as
an \emph{external check against historical causal evidence}: among
aggregations
that satisfy the gates, we prefer the one whose response profile agrees
with the only causal measurements ever published, because agreement with
external evidence is checkable and tuning to taste is not.
Section~\ref{sec:outcome} re-measures that evidence and finds it
expired on current engines; the check's expiry is itself an output of
the protocol, and nothing in the score's definition rests on it. When
the anchors and the gates conflict, the gates win, and the strategy
comparison of Section~\ref{sec:ab} shows exactly that: every candidate
that beats the selected one on anchor alignment fails a gate, and the
selection is decided by the gates. The self-measured modern anchors
(Section~\ref{sec:outcome}) close the loop: with today's near-zero
lever effects, the anchors could not define the score even if we wanted
them to; what survives, and what the gaming benchmark stress-tests, is
the constraint set.

\paragraph{Are the gates circular?} The obvious objection is that gates
chosen by the authors, at thresholds chosen by the authors, cannot
falsify anything the authors did not want falsified. Two facts bound how
far this can go. First, the selection is decided by a single binding
threshold: four of the five gate constants are cleared by every fitted
strategy with wide margins (the negative control by $-7$ points or more,
the saturation ratios well under the $3\times$ limit, length bias under
$0.12$ in absolute value), so only the dose-response threshold
discriminates, and its sensitivity is reported in full (any value in
$(-0.69, -0.36]$ yields the same selection, Section~\ref{sec:ab}). The
gates are not a free knob per candidate. Second, and more decisive, the
one property this paper headlines does not depend on the gates at all:
Lemma~\ref{lem:concave} makes bounded single-lever \emph{dose}
amplification a structural consequence of any concave transform with
$g(0) = 0$, so it holds for every member of the sqrt family whatever the
saturation threshold is set to. The gates select which aggregation
family ships; the amplification bound is a theorem about that family,
not an artifact of a threshold. Circularity would threaten a claim that
rested on passing a tuned test; the claim that survives the anchor expiry
of Section~\ref{sec:outcome} rests instead on a proved property of the
transform and on out-of-distribution attacks (Section~\ref{sec:gaming})
the calibration never generated.

\subsection{Why observational calibration fails}
\label{sec:obs}

The prior production weights were fit observationally: $n = 39$ pages
against an integer citation count from three commercial engines. That
calibration is unsound on its own terms: leave-one-out cross-validation
collapses the in-sample Spearman from $0.90$ to about $0.20$, the
strongest per-feature correlations were measured at $n \le 20$, and the
headline composite correlation carries a Fisher CI$_{95}$ of
$[0.42, 0.80]$. With $d = 11$ weights and $n < 40$ noisy observations,
observational fitting is structurally an overfitting exercise. We
therefore replace it.

\subsection{Anchors and volume control}

\citet{aggarwal2024geo} causally established the relative PAWC gain of nine
interventions. Five admit deterministic, engine-free realizations; their
published gains $\gamma \in \mathbb{R}^5$ are our anchors:

\begin{center}
\small
\begin{tabular}{lr}
\toprule
Intervention $m$ & Published gain $\gamma_m$ (\%) \\
\midrule
Quotation Addition & $+42.6$ \\
Statistics Addition & $+32.8$ \\
Cite Sources & $+27.7$ \\
Technical Terms & $+18.5$ \\
Keyword Stuffing (\textbf{negative control}) & $-8.8$ \\
\bottomrule
\end{tabular}
\end{center}

For a source $x$ and intervention $m$ at dose $\delta$, let
$x \oplus e_m^{\delta}$ denote the edited text and
$x \oplus e_0^{\delta}$ a \emph{content-neutral filler edit of the same added
length}. The volume-controlled response of any score $s$ is
\begin{equation}
\Delta_m(s) \;=\; \mathbb{E}_x\!\left[\, s(x \oplus e_m^{1}) - s(x \oplus e_0^{1}) \,\right],
\label{eq:delta}
\end{equation}
estimated over GEO-Bench sources. Volume control is essential: without it,
any length-sensitive feature inflates every edit.

\paragraph{Which feature realizes which anchor.} Because each anchor edit
must be detectable by at least one sub-component for calibration to be
meaningful, Table~\ref{tab:featanchor} reports the volume-controlled
response of every raw feature to every intervention at dose 1 (held-out
split; artifact \path{feature_anchor_matrix.json}). The realization is
sharp for four anchors: Quotation is carried almost entirely by
quotable\_density ($+47.6$ points), Statistics by statistic\_density
($+80.7$), Cite-Sources by citation\_f1 ($+82.1$) with secondary
statistic\_density and entity\_density responses, and Keyword Stuffing is
detected negatively by the entropy family (shannon $-7.4$,
information\_density $-9.1$), which is the anti-stuffing mechanism of the
score. Technical Terms is the weakly realized anchor (largest single
response $+1.0$): the feature set has no dedicated terminology component,
a candidate feature we discuss, and reject for now, in
Section~\ref{sec:ident}.

\begin{table}[t]
\centering
\footnotesize
\setlength{\tabcolsep}{5pt}
\begin{tabular}{lrrrrr}
\toprule
& \multicolumn{5}{c}{Intervention (published gain $\gamma_m$, \%)} \\
\cmidrule(lr){2-6}
Sub-component & Quot.\ ($+42.6$) & Stat.\ ($+32.8$) & Cite ($+27.7$) &
Tech.\ ($+18.5$) & KS ($-8.8$) \\
\midrule
shannon\_entropy & $-0.0$ & $0.0$ & $-0.0$ & $-0.0$ & $\mathbf{-7.4}$ \\
information\_density & $0.3$ & $2.2$ & $0.3$ & $0.6$ & $\mathbf{-9.1}$ \\
quotable\_density & $\mathbf{47.6}$ & $-3.7$ & $-2.5$ & $0.0$ & $-1.5$ \\
entity\_density & $6.8$ & $6.6$ & $12.5$ & $1.0$ & $-0.6$ \\
semantic\_coherence & $-0.5$ & $7.6$ & $-0.0$ & $-0.3$ & $-0.4$ \\
self\_containment & $0.0$ & $0.0$ & $0.0$ & $0.0$ & $-0.3$ \\
statistic\_density & $0.2$ & $\mathbf{80.7}$ & $15.5$ & $0.4$ & $-0.4$ \\
mmr\_score & $0.0$ & $0.0$ & $0.0$ & $0.0$ & $-0.1$ \\
citation\_f1 & $0.0$ & $-0.6$ & $\mathbf{82.1}$ & $0.0$ & $0.0$ \\
ndcg\_score & $0.0$ & $0.0$ & $0.0$ & $0.0$ & $-0.1$ \\
semantic\_redundancy & $0.0$ & $0.0$ & $0.0$ & $0.0$ & $-0.1$ \\
\bottomrule
\end{tabular}
\caption{Feature--anchor correspondence: mean volume-controlled response of
each raw sub-component (points) to each intervention at dose 1, test split
($n = 250$). Bold marks the dominant realization of each anchor. Quotation,
Statistics, and Cite-Sources each have a dedicated carrier; Keyword
Stuffing is detected by the entropy family; Technical Terms has no
dedicated carrier and is the weakest realized anchor.}
\label{tab:featanchor}
\end{table}

\subsection{Calibration as constrained alignment}

Let $M \in \mathbb{R}^{5 \times d}$ be the matrix of volume-controlled
sub-component responses, $M_{mi} = \mathbb{E}_x[\,g(f_i(x \oplus e_m^1)) -
g(f_i(x \oplus e_0^1))\,]$, computed on a \emph{training} split of sources.
Weights are fit by regularized cosine alignment to the anchors:
\begin{equation}
w^\star \;=\; \operatorname*{arg\,min}_{w \ge 0}\;
-\frac{\langle M w, \gamma \rangle}{\lVert M w \rVert\,\lVert \gamma \rVert}
\;+\; \lambda \lVert w - w_0 \rVert_2^2,
\qquad w^\star \leftarrow w^\star / \textstyle\sum_i w^\star_i,
\label{eq:fit}
\end{equation}
% REVIEW (D6, source fixed, regen pending): the "(v5)" label in
% Figure~\ref{fig:signal} is fixed at source: geo_paper_fig_signal.py now
% emits "fixed score (shipped weighting)". REGENERATE fig_signal.pdf
% (needs matplotlib, author env) for this footnote to hold; the compiled
% figure still shows "(v5)" until regenerated.
where $w_0$ is the frozen prior observational weighting\footnote{Internal
version labels (v3 through v6) survive only in artifact filenames; the
text refers to the prior observational weighting ($w_0$, the pre-calibration
production weights), the shipped weighting (the selected calibrated one),
and the modern-anchor recalibration (the degenerate demonstration arm).
The mapping is pinned in the reproducibility appendix.} (ridge anchor: do not
discard observationally-useful features without cause; released as
\path{weights_v3_frozen.json} since the live production file has since been
promoted to the shipped weighting). The post-hoc normalization is innocuous: the cosine term
is scale-invariant, so projecting onto the simplex leaves the fitted
alignment unchanged; only the ridge penalty is evaluated at the
unnormalized scale. The regularization strength $\lambda$ is selected on an
\emph{inner} split of the training sources (nested selection); the test
split is never touched during fitting. The unregularized fit
($\lambda \to 0$) illustrates why the ridge anchor exists: it improves
training-split alignment (cosine $0.957$ versus $0.929$ for the selected
fit) by zeroing seven of the eleven features, but generalizes worse
(held-out Pearson $0.92$ versus $0.949$) and weakens the negative control
($-2.96$ versus $-9.65$ at dose 8); it is rejected as overfit to the
anchors (per-weight analysis in Section~\ref{sec:ident}).

\subsection{Validity criteria (gates)}
\label{sec:gates}

Figure~\ref{fig:calibration} shows how fitting and falsification are
separated, and the split protocol keeps selection away from the test
data: weights are fitted on the training split (nested $\lambda$
selection on an inner split), the gates are computed on \emph{both}
splits, the selection among aggregation families is made on the
training-split gates, and the held-out test split serves as confirmation
only. All numbers reported in Table~\ref{tab:ab} are the held-out
confirmation values.

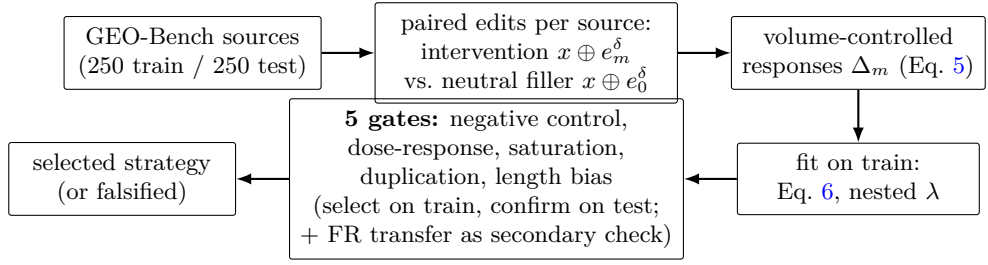
\begin{figure}[t]
\centering
\begin{tikzpicture}[node distance=4mm and 7mm]
\node[papbox, minimum width=34mm] (src) {GEO-Bench sources\\ (250 train / 250 test)};
\node[papbox, right=of src, minimum width=40mm] (edits)
  {paired edits per source:\\ intervention $x \oplus e_m^\delta$\\
   vs.\ neutral filler $x \oplus e_0^\delta$};
\node[papbox, right=of edits, minimum width=32mm] (delta)
  {volume-controlled\\ responses $\Delta_m$ (Eq.~\ref{eq:delta})};
\node[papbox, below=7mm of delta, minimum width=32mm] (fit)
  {fit on train:\\ Eq.~\ref{eq:fit}, nested $\lambda$};
\node[papbox, left=of fit, minimum width=40mm] (gates)
  {\textbf{5 gates:} negative control,\\ dose-response, saturation,\\
   duplication, length bias\\ (select on train, confirm on test;\\
   + FR transfer as secondary check)};
\node[papbox, left=of gates, minimum width=30mm] (sel)
  {selected strategy\\ (or falsified)};
\draw[paparrow] (src) -- (edits);
\draw[paparrow] (edits) -- (delta);
\draw[paparrow] (delta) -- (fit);
\draw[paparrow] (fit) -- (gates);
\draw[paparrow] (gates) -- (sel);
\end{tikzpicture}
\caption{The calibration loop. Aggregation families are fitted to the
causal anchors $\gamma$ \citep{aggarwal2024geo} on the training split; the
anchors enter only the fitting stage, and the gates are pure falsification
tests (Eq.~\ref{eq:constrained}): they are what actually decides the
selection, on the training split, with the held-out split as confirmation
(Table~\ref{tab:ab}).}
\label{fig:calibration}
\end{figure}

Alignment is the fitted \emph{objective}, reported with its uncertainty:
Pearson and Spearman correlation between $(\Delta_m)_m$ and $\gamma$,
bootstrap CI$_{95}$ over sources, and an \emph{exact} one-sided permutation
$p$-value over the $5! = 120$ orderings of the anchors (floor
$p = 1/120 \approx 0.0083$). The five \emph{gates} of
Eq.~\eqref{eq:constrained}, the pass/fail constraints behind every
``$k$/5'' in this paper, are:

\begin{enumerate}
\item \emph{Negative control.} $\Delta_{\mathrm{ks}} \le 0$ at dose 1
  \emph{and} dose 8 (falsification test of gameability).
\item \emph{Dose-response.} Volume-matched responses at doses 1, 2, 3 must
  not decline materially (worst consecutive step $\ge -0.5$ points): a user
  who adds more of a proven lever must not be punished for it.
\item \emph{Adversarial saturation.} For each additive lever, the response
  at dose 8 must not exceed $3\times$ the response at dose 1.
\item \emph{Duplication.} Full-text duplication must gain at most $+2$
  points.
\item \emph{Length bias.} $|\mathrm{corr}(s(x), \log |x|)| \le 0.35$ across
  test sources.
\end{enumerate}

A sixth criterion, \emph{cross-lingual transfer} (alignment and negative
control must not collapse on a French corpus with French edit blocks, using
the English-fitted weights, no refit), is evaluated and reported for every
strategy but kept out of the gate count: it is a secondary comparison among
gate-feasible strategies, not a feasibility constraint, and
Section~\ref{sec:ab} shows a strategy can win it while failing a gate.

\section{Strategy comparison}
\label{sec:ab}

\subsection{Design}

From one shared capture (every source scored under 21 deterministic
variants), we compare three aggregation families, each re-fit by
Eq.~\eqref{eq:fit} on the training split: \textbf{A} raw-linear
($g = \mathrm{id}$, the production family); \textbf{B} sqrt-saturated
(Eq.~\eqref{eq:content}); \textbf{C} quantile-linear
($g_i = 100 \cdot \widehat{\mathrm{ECDF}}^{\,\mathrm{train}}_i$). Baselines:
production weights, equal weights, best-single-feature, and the prior
intervention-calibrated raw-linear candidate. Corpora: 500 English GEO-Bench
sources (250 train / 250 test, fixed split); 60 French sources for transfer.
All numbers in this section use the corrected negative-control edit and
regenerated captures of Section~\ref{sec:ksbug}
(\path{ab_capture_en_500_ksfix.jsonl}, \path{ab_strategies_report_ksfix.json}).

\subsection{Results}

\begin{table}[t]
\centering
\footnotesize
\setlength{\tabcolsep}{3pt}
\begin{tabular}{lcccccccc}
\toprule
Strategy & Pearson [CI$_{95}$] & $\rho$ & perm.\ $p$ & KS$_{8\times}$ &
dose & game & len $r$ & gates \\
\midrule
\textbf{B sqrt (selected)} & 0.949 [0.932, 0.964] & \textbf{1.0} &
\textbf{0.0083} & $-9.7$ & $-0.36$ & 0.96 & $0.00$ & \textbf{5/5} \\
A raw-linear & 0.956 [0.945, 0.966] & 1.0 & 0.0083 & $-12.6$ & $-0.74$ & 0.91 & $+0.01$ & 4/5 \\
obs.\ recalib.\ (raw) & \textbf{0.995} [0.992, 0.997] & 1.0 & 0.0083 & $-18.5$ & $-0.69$ & 0.81 & $-0.09$ & 4/5 \\
obs.\ recalib.\ (aggr.) & 0.971 [0.960, 0.981] & 1.0 & 0.0083 & $-17.1$ & $-1.21$ & 0.91 & $-0.07$ & 4/5 \\
C quantile & 0.882 [0.871, 0.891] & 0.7 & 0.0667 & $-10.3$ & $\mathbf{-2.27}$ & 0.68 & $-0.12$ & 4/5 \\
Prior observational & 0.713 [0.697, 0.734] & 0.6 & 0.10 & $-18.6$ & $-0.27$ & 1.00 & $+0.08$ & 5/5 \\
Equal weights & 0.737 [0.709, 0.766] & 0.6 & 0.0917 & $-7.3$ & $-0.69$ & 0.85 & $+0.08$ & 4/5 \\
Best single feat. & 0.896 [0.895, 0.896] & 0.9 & 0.0583 & $-36.0$ & $-0.03$ & n/a & $\mathbf{-0.38}$ & 3/5 \\
\bottomrule
\end{tabular}
\caption{Strategy comparison: held-out confirmation values (test split,
$n = 250$ sources, volume-controlled, corrected negative-control edit);
selection among fitted families was made on training-split gates, with an
identical outcome (Section~\ref{sec:gates}). KS$_{8\times}$:
keyword-stuffing response at dose 8 (gate: $\le 0$). dose: worst
consecutive dose step (gate: $\ge -0.5$). game: max ratio of dose-8 to
dose-1 response (gate: $\le 3$; values $< 1$ indicate saturation). len
$r$: length-bias correlation. Best single feature is shannon-only, which
fails the saturation gate (undefined ratio: its dose-1 response is
essentially zero) and the length-bias gate ($r = -0.38$).}
\label{tab:ab}
\end{table}

Table~\ref{tab:ab} carries the main message: anchor alignment does not
decide the selection, the gates do. The three raw-linear fits and the
selected strategy all sit at the exact permutation floor ($p = 0.0083$;
the quantile fit does not, $p = 0.0667$), and two raw-linear candidates
\emph{beat} the selected strategy on alignment (the observational recalibration $0.995$,
fitted raw $0.956$ versus $0.949$); both fail the dose-response gate,
declining materially under repeated statistics dosing (worst step $-0.69$
and $-0.74$). The quantile family fails dose-response outright ($-2.27$).
The selected concave family is the only fitted strategy passing all five
gates. A score that punishes a user for adding more of a causally-proven
lever is miscalibrated no matter how well it aligns at dose 1, and a
selection by alignment alone would have picked a strategy with exactly
that defect. The dose-response threshold itself ($-0.5$ points) is the one
gate constant the selection is sensitive to, so we report its sensitivity:
any threshold in $(-0.69, -0.36]$ yields the same selection, the
alternatives failing at $-0.69$, $-0.74$, $-1.21$, and $-2.27$ while the
selected strategy's worst step is $-0.36$. All other gate constants were
fixed in the harness before any strategy comparison and are not binding:
every fitted strategy passes the negative control with margins of $-7$
points or more, the per-strategy saturation ratios (Table~\ref{tab:ab},
\texttt{game} column) sit at $0.68$ to $1.0$ against a threshold of $3$,
and length-bias correlations stay below $0.12$ in absolute value except
for the shannon-only baseline ($-0.38$).

Because the gates decide the selection, the selection itself must not
touch the test split. Recomputing all six criteria on the training split
(artifact \path{ab_gates_trainsplit_report.json}) gives an identical
outcome: the sqrt family is the only fitted strategy at 5/5 on the
training split (Pearson $0.946$), the raw and quantile families fail the
same dose-response gate there, and the per-gate pass/fail pattern is
identical across splits for all three fitted families. The held-out
values of Table~\ref{tab:ab} are therefore confirmation, not selection.
The confirmation step is not vacuous: one fixed baseline (shannon-only)
passes all five gates on the training split yet fails two on the test
split, exactly the overfitting mode a held-out confirmation exists to
catch. The production weighting is the mirror
image of the fitted raw family: robust (5/5 gates) but misaligned
($0.713$, $p = 0.10$, not
significant at the anchor level), ranking quotation, the strongest
published lever, last among the additive levers ($+1.1$ points versus
$+5.5$ for statistics and $+6.7$ for cite-sources). Notably the quantile
family has the best French transfer ($0.916$,
\path{ab_strategies_report_ksfix.json}), again illustrating why a
single-number comparison would mislead.

\subsection{Intervention testing as an audit instrument: two bugs found
and fixed}
\label{sec:frfix}

The first French run exposed a coverage hole invisible to any observational
correlation: the quotable-density patterns were purely definitional, so
attributed quoted speech (the literal realization of Quotation Addition)
matched nothing; the French response of the feature to the quotation edit was
\emph{negative} ($-1.9$ points: added quotes increased the sentence count
without matching). English worked partly by accident. After adding
quoted-speech-with-attribution patterns in both languages, the feature
response is $+43.7$ and the selected strategy's French transfer rises from
Pearson $0.59$ (pre-fix capture, artifact
\path{ab_strategies_report.json}) to $0.836$ on the final corrected
captures. We propose intervention testing as a general, zero-cost unit-test
harness for content features.

\paragraph{The audit cuts both ways: a bug in the negative control
itself.}
\label{sec:ksbug}
Preparing the formal specification of Appendix~\ref{app:features} for this
paper, we audited the intervention harness with the same discipline and
found a bug in the \emph{negative-control edit}: a double string join in
the keyword-stuffing implementation space-separated every character, so the
appended block was single-letter noise (``t r a n s f o r m e r \dots'')
rather than repeated keywords. Every previously reported stuffing response,
including the earlier falsification of the quantile family (a $+3.6$
positive control response), was measured against letter noise, not keyword
stuffing. We fixed the edit, regenerated the stuffing variants of both
captures, and re-ran every downstream analysis; all numbers in this paper
are post-fix. Three things changed and one did not. The falsification of
the quantile family changed in kind: under real keyword stuffing its
negative control is comfortably negative ($-10.3$), and it now fails the
dose-response gate instead (Table~\ref{tab:ab}). The magnitude ordering of
the stuffing penalty across transforms reversed (Section~\ref{sec:robust}).
The headline selection did not change: the concave family remains the only
fitted strategy passing all five gates. We report the bug rather than
silently absorbing it for the same reason we report the negative results of
Section~\ref{sec:outcome}: the value of a deterministic, fully released
harness is precisely that such bugs are findable, attributable, and cheap
to correct (the full re-run costs minutes of CPU), whereas an answer-side
evaluation with the same defect would be unauditable in retrospect.

\paragraph{Robustness and production check (summary).} The selection
is stable: across 10 random splits, plus-minus 10\% weight
perturbations, the lambda grid, and the transform exponent, held-out
alignment stays in a narrow plateau and the negative control stays
strongly negative; an end-to-end replication inside the production
pipeline (real scorer, shipped: $-4.0$ at dose 1 and $-11.2$ at dose 8
in English; prior: $-5.5$/$-18.3$) confirms the harness numbers.
Details and artifacts: Appendix~\ref{app:robfull}.

\subsection{Weight identifiability (summary)}
\label{sec:ident}

Full machinery, profile grids, and the per-weight discussion are in
Appendix~\ref{app:identfull}; the body keeps the conclusion the rest of
the paper relies on. Per-weight profile refits over the training split,
evaluated held-out, separate the eleven weights into three regimes
(Table~\ref{tab:ident}, Figure~\ref{fig:profiles}): two weights pinned
by the anchors (quotable and statistic density), two forced nonzero by
the gates (the entropy pair, which carries the stuffing response), and
the remainder admissible over wide intervals, i.e.\ supplied by the
prior rather than identified by evidence, and labeled as such. The
unregularized fit is the cautionary tale: it improves training alignment
by concentrating mass on four features and gives up most of the
stuffing response, which the gates veto.

\begin{table}[t]
\centering
\footnotesize
\setlength{\tabcolsep}{4.5pt}
\begin{tabular}{lccccl}
\toprule
Sub-component & $w$ (released) & LOFO Pearson & LOFO gates & Admissible
$w_i$ & Role \\
\midrule
quotable\_density & 0.124 & $\mathbf{0.688}$ & 5/5 & $[0.12, 0.16]$ &
alignment carrier \\
statistic\_density & 0.051 & $0.921$ & \textbf{4/5} (dose) & $[0.02, 0.05]$
& alignment + dose \\
citation\_f1 & 0.019 & $0.947$ & 5/5 & $[0, 0.02]$ & capped carrier \\
shannon\_entropy & 0.246 & $0.939$ & 5/5 & $[0.02, 0.50]$ & control mass \\
information\_density & 0.244 & $0.930$ & 5/5 & $[0.12, 0.50]$ & control mass \\
entity\_density & 0.094 & $0.945$ & 5/5 & $[0, 0.50]$ & diffuse support \\
semantic\_coherence & 0.084 & $0.950$ & 5/5 & $[0, 0.12]$ & bounded above \\
self\_containment & 0.073 & $0.949$ & 5/5 & $[0, 0.50]$ & weakly identified \\
mmr\_score & 0.036 & $0.948$ & 5/5 & $[0, 0.50]$ & weakly identified \\
ndcg\_score & 0.015 & $0.949$ & 5/5 & $[0, 0.50]$ & weakly identified \\
semantic\_redundancy & 0.015 & $0.949$ & 5/5 & $[0, 0.50]$ & weakly identified \\
\bottomrule
\end{tabular}
\caption{Per-weight identifiability on the corrected capture (reference
refit: held-out Pearson $0.949$, 5/5 gates). LOFO: held-out alignment and
gates after refitting with $w_i = 0$. Admissible $w_i$: outermost grid
values where the profile refit keeps Pearson $\ge 0.94$ and 5/5 gates
(grid resolution $0.02$--$0.05$; the interval is a per-coordinate slice
through the admissible set, not a product box). The released vector of
Table~\ref{tab:features} jointly passes all gates on this capture with
Pearson $0.970$.}
\label{tab:ident}
\end{table}

\begin{figure}[t]
\centering
\includegraphics[width=\linewidth]{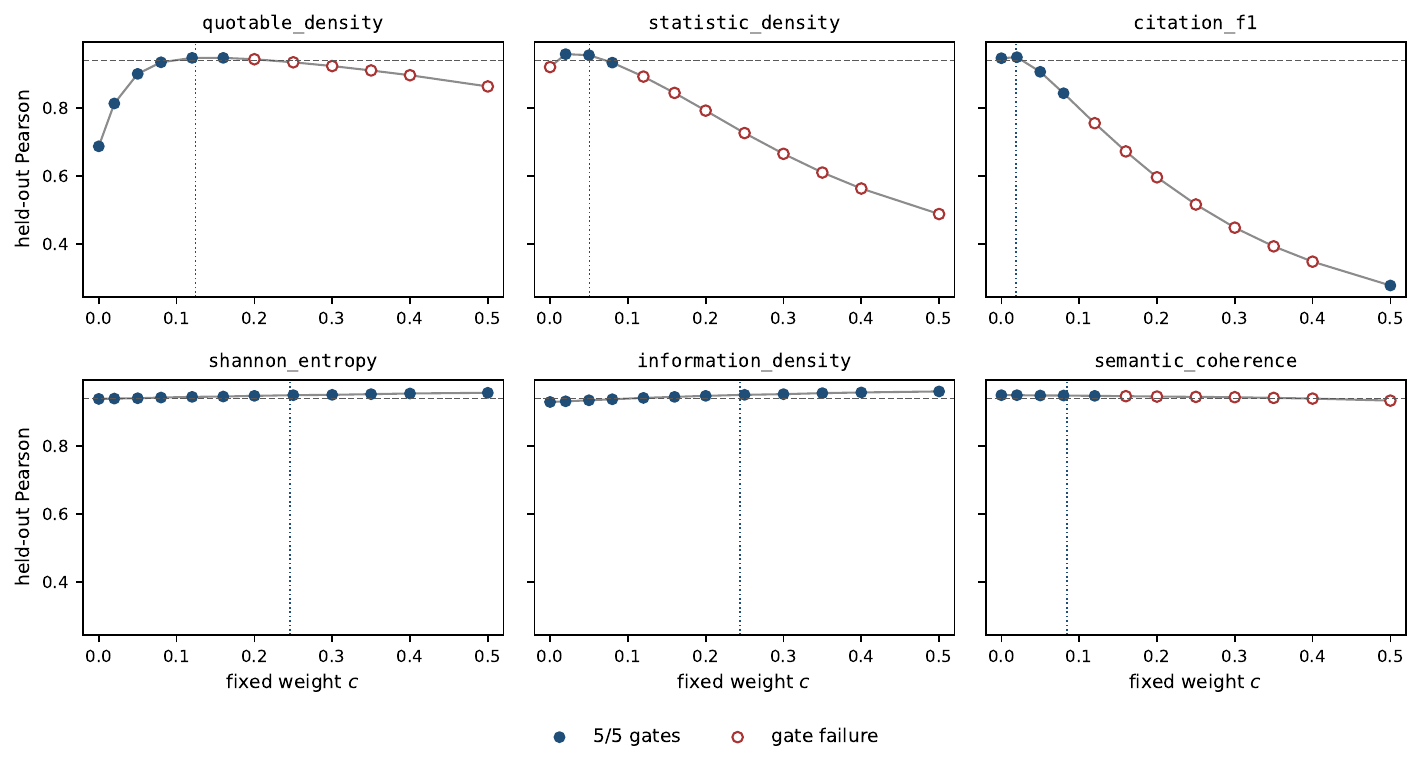}
\caption{Profile refits for six sub-components
(\texttt{identifiability\_report.json}, script
\texttt{geo\_paper\_fig\_profiles.py}): held-out anchor Pearson when the
feature's weight is fixed at $c$ and the remaining weights are refit.
Filled points pass all five gates; open red points fail at least one.
Dashed line: the $0.94$ split-stability floor; dotted vertical line: the
released weight. quotable\_density and statistic\_density are pinned by
the data; the entropy pair is flat above a lower bound; citation\_f1 is
capped from above.}
\label{fig:profiles}
\end{figure}

\section{The score as a manipulation filter: a gaming-detection benchmark}
\label{sec:gaming}

The properties the preceding sections establish, deterministic and
auditable computation, falsification gates, a provable single-lever
amplification bound (Lemma~\ref{lem:concave}), are the properties of a
\emph{content-manipulation filter} in the web-spam lineage
\citep{ntoulas2006spam, gyongyi2004trustrank}, not of a citation
predictor. This section evaluates that use directly, on a benchmark built
entirely from the released capture and scorer at zero marginal cost
(script \path{geo_paper_gaming_detection.py}, artifact
\path{gaming_detection_report.json}).

\paragraph{Benchmark.} The clean class is the 500 unmodified EN GEO-Bench
sources plus their length-matched neutral-filler variants at every
released dose: a usable filter must not flag benign volume growth. The
gamed class covers three attack families: \emph{degradations} (keyword
stuffing at dose 1 and 8, duplication), \emph{lever amplification}
(quotation, statistics, cite-sources stuffed at dose 8, i.e.\ an attacker
pushing the score's own calibrated levers), and four \emph{combined}
attacks generated with the released edit machinery (composition order
fixed; 2{,}000 additional deterministic scorings, cached in
\path{gaming_combo_capture.jsonl}). Detection is score-based with the
flagging threshold set at $5\%$ false-positive rate on the clean pool
only; separation is reported as AUC against the dose-matched neutral
control, so every comparison stays volume-controlled. Both the shipped weighting and the prior observational weighting are evaluated.

\begin{table}[t]
\centering
\footnotesize
\setlength{\tabcolsep}{4pt}
\begin{tabular}{lcccc}
\toprule
Attack & $\Delta$score (shipped) & AUC (shipped) & det@5\%FPR & AUC (prior) \\
\midrule
keyword stuffing, d1 & $-2.8$ & $0.620$ & $0.124$ & $0.771$ \\
keyword stuffing, d8 & $-13.0$ & $0.851$ & $0.498$ & $\mathbf{0.997}$ \\
duplication & $-2.9$ & $0.529$ & $0.114$ & $0.449$ \\
stuffing $+$ duplication, d8 & $-16.4$ & $\mathbf{0.944}$ & $0.788$ & --- \\
stuffing $+$ quotation, d8 & $-3.1$ & $0.569$ & $0.002$ & --- \\
\midrule
quotation, d8 & $+3.5$ & $0.208$ & $0.000$ & $0.664$ \\
statistics, d8 & $+0.1$ & $0.394$ & $0.000$ & $0.516$ \\
cite-sources, d8 & $+0.5$ & $0.372$ & $0.000$ & $0.279$ \\
quote$+$stats$+$cite, d8 & $+4.5$ & $0.147$ & $0.000$ & --- \\
\bottomrule
\end{tabular}
\caption{Gaming detection on 500 sources per class, in-distribution
attacks. AUC is against the
dose-matched neutral control (volume-controlled); the detection threshold
is fixed at $5\%$ FPR on the clean pool. Degradation attacks (top) are
separable at strong dose; amplification of the score's own levers
(bottom) is not detected and is mildly rewarded, within the bounded gain
quantified in the text. Out-of-distribution attacks and web-spam
baselines are in Table~\ref{tab:gaming2}. Prior-weighting AUCs for the
released single attacks; combo
variants are scored under both weightings in the artifact.}
\label{tab:gaming}
\end{table}

\begin{table}[t]
\centering
\footnotesize
\setlength{\tabcolsep}{4pt}
\begin{tabular}{lccc}
\toprule
Attack & score AUC & rep.-rate AUC & near-dup AUC \\
\midrule
keyword stuffing, d8 (ID) & $0.851$ & $\mathbf{1.000}$ & $0.988$ \\
duplication (ID) & $0.529$ & $0.670$ & $\mathbf{0.685}$ \\
link injection, heavy (OOD) & $0.468$ & $\mathbf{0.859}$ & $0.629$ \\
keyword interleave, heavy (OOD) & $0.440$ & $\mathbf{0.883}$ & $0.315$ \\
boilerplate padding, heavy (OOD) & $0.415$ & $0.521$ & $0.415$ \\
\bottomrule
\end{tabular}
\caption{Hardened benchmark: the score against two no-fitting web-spam
baselines, on in-distribution (ID) and out-of-distribution (OOD)
attacks. Baselines dominate detection wherever detection is possible;
the score detects none of the OOD attacks; boilerplate padding evades
everything tested. The score's distinctive property is the bounded
attacker gain of Result 1, not detection
(\texttt{gaming\_detection\_v2\_report.json}).}
\label{tab:gaming2}
\end{table}

\paragraph{Result 1 (the headline property): the attacker's gain is
bounded, dose-decreasing, and sub-additive.} An attacker amplifying the
score's own positive levers is not flagged, and gains score, but the
gain is capped: the best any single lever achieves at any dose is
$+6.1$ points (quotation at dose 1), \emph{falling} to $+3.5$ at dose 8;
statistics and cite-sources decay from $+4.6$ and $+3.9$ to $+0.1$ and
$+0.5$; the three-lever combination at dose 8 yields $+4.5$,
sub-additive (Figure~\ref{fig:gaming}, left). This is the empirical
counterpart of Lemma~\ref{lem:concave}, and it is a property the
shipped concave weighting owns: under the linear prior weighting the
dose curves are not uniformly capped in the same way (quotation at dose
8 goes negative while cite-sources retains $+2.6$), and no web-spam
baseline bounds anything, baselines only detect. A score that an
optimizer cannot inflate by more than a few points, while remaining
auditable line by line at zero cost, is the artifact's one distinctive
property, and it is both proved and measured. One operational
question the benchmark sharpens rather than leaves open. The leaderboard
shrinkage attenuates a page-level gain before it can move a ranking:
under Definition~2, inflating a single page of an $n$-page domain by
$\Delta$ raises the domain's shrunk score by exactly $\Delta/(k + n)$, at
most $1$ point at the $\Delta = 6$ cap (attained at $n = 1$, decaying as
$1/(k + n)$ thereafter), whereas inflating \emph{every} page by $\Delta$
raises it by $\Delta\, n/(k + n) < \Delta$, approaching the raw gain only
for large, fully-compromised domains ($\tilde{s}$ is affine in $\bar{s}$
with slope $n/(k + n)$, and one page shifts $\bar{s}$ by $\Delta/n$). What
a released snapshot cannot yet resolve is the last step, from a bounded
score change to a change in leaderboard \emph{position}, which depends on
how densely domains cluster near the attacked one; we leave that
measurement as future work because it requires a frozen production
leaderboard outside the released GEO-Bench artifacts.
% TO-PRODUCE (authors): the positional step only. Simulate the
% single-page and whole-domain inflation above through Definition-2
% shrinkage on a frozen (hashed) production leaderboard snapshot and
% report the rank-displacement distribution (median, p95, max; %% of
% domains gaining >=1 and >=5 positions), stratified by n. The score-space
% bound is now proved in-text; only the points->positions map needs the
% snapshot. Full spec in REVIEW_REPORT_v9.md.

\paragraph{Result 2: the division of labor: baselines detect, the
score bounds.} Within the calibration's own attack families, heavy
stuffing separates from matched neutral filler at AUC $0.851$ ($49.8\%$
detection at $5\%$ FPR) and stuffing-plus-duplication at $0.944$
(Table~\ref{tab:gaming}). Three attributions follow from the
hardened benchmark (\path{gaming_detection_v2_report.json}). First,
standard web-spam baselines in the \citet{ntoulas2006spam} lineage
dominate detection wherever detection is possible: a no-fitting
token-repetition-rate detector reaches AUC $1.000$ on heavy stuffing and
a shingle-based near-duplicate detector $0.988$ (the score: $0.851$);
on pure duplication the near-duplicate detector reaches $0.685$ against
the score's chance-level $0.529$. Second, the prior observational
weighting detects heavy stuffing better than the shipped one ($0.997$
versus $0.851$): detection strength is chiefly a feature-set property,
not a calibration achievement. Third, out-of-distribution attacks that
the calibration harness never produced (spam link injection, in-text
distributed keyword interleaving, SEO boilerplate padding, each
deterministic and released) are \emph{not detected by the score at
all}: AUC $0.41$ to $0.52$ against matched neutral filler, chance
level. The repetition-rate baseline catches two of the three (in-text
interleaving $0.88$, link injection $0.86$ heavy-dose); boilerplate
padding evades the score and both baselines ($\le 0.59$). Detection by
this score does not generalize beyond the perturbation families it was
gated against, and we state that as a limit rather than argue around
it.

\paragraph{Positioning: what the filter is, and what it is not.} The
deployable configuration that follows from the measurements is layered:
cheap standard baselines (repetition rate, near-duplicate shingling) for
degradation-spam \emph{detection}, and the calibrated score as the
\emph{bounded quality measure} that an optimizer cannot inflate by more
than a few points, for corpus filtering, editorial quality control, and
regression testing of content pipelines. The score's within-query
citation correlation (Section~\ref{sec:outcome}) is a secondary
validity signal for this use. The benchmark does not support calling
the score a manipulation \emph{detector}: its distinctive, proved
property is being hard to \emph{inflate}, not flagging every spam
family, and the two claims should not be traded for one another.

\begin{figure}[t]
\centering
\includegraphics[width=0.95\linewidth]{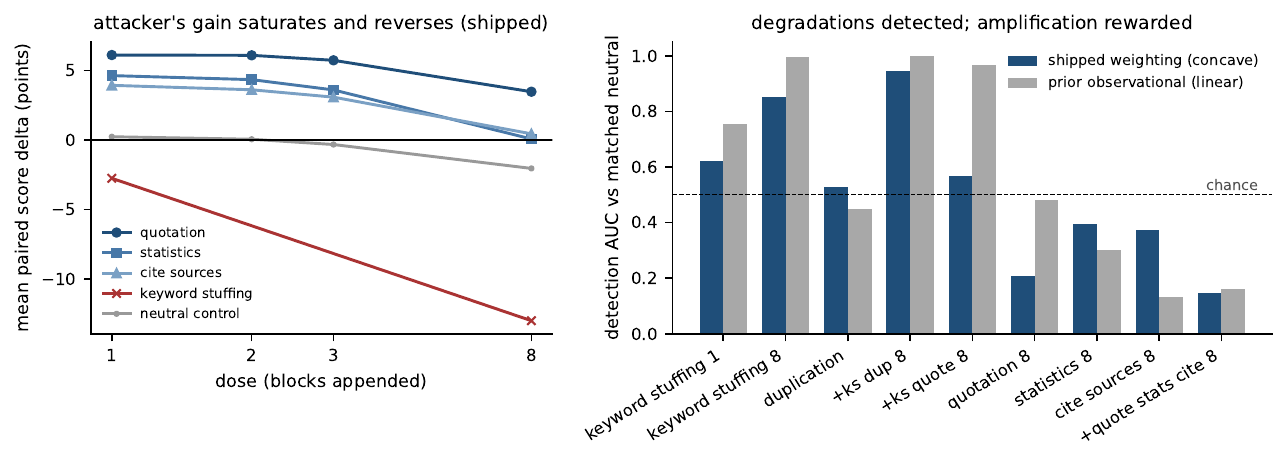}
\caption{Gaming-detection benchmark. Left: the attacker's mean paired
score gain by dose under the shipped weighting; positive levers saturate and decay while
degradations go strongly negative; the neutral control stays flat.
Right: detection AUC against matched neutral filler per attack, shipped
and prior weightings. Script \texttt{geo\_paper\_fig\_gaming.py}.}
\label{fig:gaming}
\end{figure}

\section{Outcome validation with modern answer engines}
\label{sec:outcome}

Construct validity certifies alignment with causal levers, not the visibility
outcome itself. We therefore close the loop with a paired, volume-controlled
outcome experiment against open-weights answer engines (temperature 0, fixed
seed, pinned model identifiers), at zero marginal API cost.

\paragraph{Protocol.} For each GEO-Bench query with its five sources, the
engine answers with inline per-sentence citations under three conditions:
unmodified sources; the official target source with the deterministic
quotation edit; and the same target with length-matched neutral filler
(volume control in outcome space). Visibility is a PAWC-like share: the
position-weighted word count of sentences citing each source, normalized over
sources. The primary estimand is the paired difference
$\Delta\pawc_{\mathrm{quote}} - \Delta\pawc_{\mathrm{neutral}}$ on the target
source (mean, bootstrap CI, exact sign test). Secondarily, we report the
within-query rank correlation and the pairwise concordance between each
source's content score and its baseline PAWC share, for both the production
and the selected weighting.

{\sloppy
We run this across seven engine families: six open-weights
(\texttt{mistral-medium-3.5}, \texttt{gemma-4-31b}, \texttt{nex-n2-pro},
\texttt{gpt-oss-120b}, \texttt{qwen3.5-122b}, \texttt{llama-3.3-70b}) and
one closed-weights commercial model (\texttt{gemini-3.1-flash-lite}):
821 (engine, query) outcome records over 229 unique queries, of which 819
form complete five-source ranking groups (Section~\ref{sec:ltr}) and 777
have non-degenerate within-query variance for the correlation
analysis. A July-2026 replication on three additional proprietary
engines (\texttt{gpt-5.5}, \texttt{gpt-5.4}, \texttt{gpt-5.4-mini};
$n = 150$ queries each) is reported at the end of this section and kept
out of the pooled numbers above, which predate it.\par}

\paragraph{Observational: the score predicts citation (robust).}
% NUMBERS-FROM: outcome_stats_report.json (geo_paper_outcome_stats.py over outcome_runs_*.jsonl)
The score's pooled within-query Spearman against citation is
$0.114$\footnote{Three within-query Spearman values for the fixed score
recur and are not interchangeable: $0.114$ pools the seven original arms
with non-degenerate variance ($n = 777$, this section); $0.119$ is the
ten-arm ranking-study value ($n = 1{,}269$ groups, Table~\ref{tab:ltr};
harness value $0.1188$); $0.122$ is the same quantity under the baseline
script's group filtering. We cite $n$ with each occurrence.}
($n = 777$, $t = 6.12$, $p = 1.5\times10^{-9}$), individually
significant on four of seven original arms, near zero on gpt-oss
($0.016$); per-arm values are released
(\path{outcome_stats_report.json}). The prior observational weighting
pools slightly higher ($0.137$), a point the re-weighting note below
owns. The effect is modest, as expected for ranking five already
topically-relevant sources from text alone, but it rejects the
no-signal null decisively, and it is a \emph{secondary} validity signal
for the filter positioning of Section~\ref{sec:gaming}, not a headline.

\paragraph{Error analysis: where the score fails.} The failures are not
random; they concentrate where the score cannot separate the candidates.
Grouping the 1{,}213 correlation-eligible records (ten arms, gpt-5.x
replication included) by the within-query standard
deviation of the five candidates' scores, the mean within-query
Spearman rises monotonically from $0.066$ in the most compressed quartile
to $0.240$ in the most dispersed (artifact
\path{ltr_strict_report.json}); in the 92 groups where the five
candidates take three or fewer distinct score values, the mean correlation
is $0.01$, indistinguishable from chance, against $0.13$ elsewhere: the
ceiling saturation of Appendix~\ref{app:diagnostics} degrades
within-query ordering into tie-breaking noise. Two disciplined attempts
to act on this are released as negative results. A dispersion-based
confidence flag failed held-out confirmation twice (inversion at median
coverage; non-significant at quartile coverage in a disclosed second
attempt): dispersion is a descriptive diagnostic, not a calibrated
confidence signal (\path{dispersion_confidence_report.json}). Percentile
renormalization of the saturated features removes the ceiling and
preserves alignment ($0.951$) but breaks the dose-response gate and
slightly lowers the outcome correlation ($0.106$ vs $0.115$); per the
pre-registered rule the features stay
(\path{hybrid_renorm_report.json}).

\paragraph{Freshness-blend ablation.} Re-scoring every outcome record
with and without the $8\%$ freshness blend (recomputed scores match the
released records to $0.005$), the blended score pools at $0.132$ against
$0.114$ for the content score alone, a paired within-query difference of
$+0.018$ ($t = 3.28$, $p = 0.001$), positive on six of the seven
original arms (\path{freshness_ablation_report.json}). The blend is retained on
that observational evidence and on continuity with the deployed
configuration; its causal status is untested, no anchor targets
recency, and excluding it from the shipped score is a defensible
alternative we did not take.

Arm-inclusion criterion: the analyzed arms are those with complete
three-condition coverage at target scale. Three abandoned partial arms,
a 13-record pilot, and three credit-interrupted Claude arms (47/41/13
queries) are released but excluded by the same rule; none is excluded
for its result.

\paragraph{Interventional: the 2023 GEO toolbox does not transfer.}
This is the paper's central outcome result, and it is a different kind of
claim from the correlation above: the observational paragraph asks whether
the score \emph{tracks} citation; this paragraph asks whether the published
causal \emph{levers} still move it. Table~\ref{tab:transfer} summarizes.
We test the deterministic realizations of the three interventions the GEO
paper reports as most effective: Quotation, Statistics, and Cite-Sources
($+42.6$, $+32.8$, $+27.7$ published PAWC gain).

\begin{table}[t]
\centering
\footnotesize
\setlength{\tabcolsep}{4pt}
\begin{tabular}{lcccl}
\toprule
Intervention & Published gain & Measured effect ($n = 605$) &
gpt-5.x repl.\ ($n = 450$) & Conclusion \\
\midrule
Quotation & $+42.6\%$ & $-0.0009$ ($p = 0.82$) &
$-0.0106$ ($p = 0.04$ unc.) & does not transfer \\
Statistics & $+32.8\%$ & $-0.0002$ ($p = 0.96$) &
$-0.0066$ ($p = 0.25$) & does not transfer \\
Cite-Sources & $+27.7\%$ & $-0.0101$ ($p = 0.04$ unc., $0.13$ Bonf.) &
$-0.0054$ ($p = 0.34$) & does not transfer \\
\bottomrule
\end{tabular}
\caption{Causal transfer of the three strongest 2023-era GEO levers to
modern engines: published PAWC gains versus the paired, volume-controlled
effect measured across six engine families plus the closed commercial
model ($n = 605$), and its July-2026 replication on three proprietary
gpt-5.x engines ($n = 450$). None moves citation in the published
direction; nominally significant effects are \emph{negative}, and no
replication-column effect survives Bonferroni correction over the three
pooled lever tests. The measured columns are PAWC \emph{share}
differences (a fraction); the self-measured modern-anchor paragraph below
rescales the same quantity to percentage points ($\times 100$) over a
larger arm set ($n = 1{,}087$ to $1{,}531$), so e.g.\ cite-sources
$-0.0101$ here and $-0.79$ percentage points there are the same effect on
different subsets, not a discrepancy.}
\label{tab:transfer}
\end{table}

The paired quotation-vs-neutral effect is null on six of the seven
original engines and nominally negative on gpt-oss ($-0.055$,
$p = 0.013$ uncorrected, n.s.\ after Bonferroni). On the six-engine
subset with full three-lever coverage ($n = 605$), none of the three
levers moves citation in the published direction (Table~\ref{tab:transfer});
the only nominally significant pooled effect (cite-sources $-0.0101$,
$p = 0.04$ uncorrected) is \emph{negative}. An instructive scale detail:
an interim $n = 41$ arm showed a suggestive positive quotation effect
($+0.027$) that washed out to $+0.005$ by $n = 100$. We therefore do not
claim any strong 2023-era lever causally moves citation on a modern
engine, consistent with \citet{puerto2025cseo}. Two scope notes bound
the claim itself: the test regime (five curated sources, a fixed
citation-format prompt, a PAWC-like share) is a controlled proxy, so
part of the non-transfer could in principle be an artifact of the
simplified frame rather than of the engines; and the constraint-set
validity of Sections~\ref{sec:method} and~\ref{sec:gaming} stands on
its own and does not depend on this causal claim.

\paragraph{Replication on proprietary frontier engines (July 2026).}
% NUMBERS-FROM: outcome_multi_gpt55/gpt54/gpt54mini.jsonl
Against the residual objection that open-weights and budget commercial
engines might cite differently from the frontier, the full three-lever
protocol was rerun on \texttt{gpt-5.5}, \texttt{gpt-5.4}, and
\texttt{gpt-5.4-mini}. These arms are single-run by design: the
provider exposes neither temperature nor seed, so
replication-by-reseeding is not available to anyone; we buy statistical
power through scale instead ($n = 150$ paired queries per arm, 450 in
the pooled test, reasoning disabled) and release every raw record,
which makes the arms exactly reproducible in the only sense the
platform permits, from the released data. Both halves replicate: no lever moves citation in the
published direction (pooled quotation $-0.0106$, nominally negative
like the gemini arm; statistics $-0.0066$; cite-sources $-0.0054$;
Table~\ref{tab:transfer}), and the score's pooled Spearman on the three
arms is $0.118$ ($n = 450$, $p = 1.0\times10^{-6}$), an out-of-family
replication of the $0.114$ headline (per arm $0.080/0.118/0.157$).
Raw records and per-arm reports are released.

\paragraph{Self-measured modern anchors: closing the loop the paper
promised.} The transfer result raises the question this paper previously
deferred to future work: what does a causal anchor set measured on
\emph{modern} engines look like? The released records answer it at zero
additional cost. Pooling the paired, volume-controlled effect of each
deterministically realizable lever across every complete arm gives, in
PAWC percentage points: quotation $-0.33$ (95\% CI $[-0.84, +0.17]$,
$n = 1{,}531$), statistics $-0.28$ ($[-0.97, +0.43]$, $n = 1{,}087$),
cite-sources $-0.79$ ($[-1.53, -0.14]$, $n = 1{,}087$), against published
2023 values of $+42.6$, $+32.8$, $+27.7$. Two of the three intervals
cover zero and the only one that excludes it is \emph{negative}: there is
no positive modern anchor set to recalibrate toward
(\path{modern_anchors_report.json}).

We nonetheless ran the recalibration as a demonstration, as a clearly
labeled additional arm (the modern-anchor recalibration, released in
\path{weights_v6_modern.json}, not shipped): identical pipeline (sqrt
family, ridge to the pinned $w_0$, nested $\lambda$, gates on train,
held-out confirmation), anchor vector replaced by the self-measured one.
The outcome is instructive in exactly the way a degenerate fit should be:
$\lambda$ selects the maximum of the grid (the most regularized option),
held-out gates drop to 4/5, and the exact permutation floor at three
anchors ($1/3! = 0.167$) makes every alignment claim untestable. The
unregularized variant is starker: it zeroes precisely the lever-carrying
features (quotable density, statistic density, citation F1, entity
density), i.e.\ a modern-anchored calibration guts the score of the
components that respond to interventions at all
(\path{identifiability_report_v6.json}). The honest conclusion, which
reframes Section~\ref{sec:method} rather than undermining it: what
calibration buys on today's engines is not anchor recency, it is the
response-surface discipline the gates enforce (bounded amplification,
negative control, dose sanity), and those properties are exactly the
ones the gaming benchmark of Section~\ref{sec:gaming} shows surviving.
The shipped weighting, for completeness, aligns with the modern vector at $0.744$
under the same untestable floor, and keeps 5/5 gates.

\paragraph{What this section establishes, and what it cannot.} The
query-agnostic score carries real but bounded outcome signal. This is not a
deficiency to be tuned away: a score that never sees the query
\emph{cannot}, even in principle, resolve which of five topically relevant
sources answers a particular question. Having quantified exactly how far
content-only signal goes, we next quantify how much of the remaining gap is
closed by modeling the query explicitly.

\section{The query-conditioned ceiling, in brief}
\label{sec:ltr}

This section exists to locate the ceiling that forces the paper's
positioning; the full study (feature stack, baselines, dependency
ablations, leave-one-engine-out transfer, per-tag breakdowns) is
Appendix~\ref{app:ltrfull}. Everything uses cross-validation folds drawn
over unique queries. The 229 unique queries are a property of
GEO-Bench's test split, not a sampling choice of ours; we treat the
count as a design constraint, which is why the paper claims no
fitted-model superiority at this scale, draws every fold over unique
queries, and reports multiseed spreads rather than single-seed points.
Three facts carry the argument. First, the ceiling
is real and the score sits at it: a LambdaMART free to combine the same
content features nonlinearly reaches held-out within-query Spearman
$0.104$ ($0.088 \pm 0.029$ across seeds) against $0.119$ for the fixed
score; ridge reaches $0.075$. No content-only model we can fit beats the
fixed, auditable score, the empirical counterpart of
Proposition~\ref{prop:ceiling}. Second, the query recovers most of the
rest: the strongest single query-conditioned signal is an open-weights
cross-encoder at $0.388$ with no fitting, which no fitted model reliably
beats at 229 unique queries (full model $0.348 \pm 0.013$); conditioning
roughly triples the measured signal (Figure~\ref{fig:signal}). Third,
the honest caveat that makes the first two credible: our first version
of this evaluation drew folds over (engine, query) records, and because
150 of 229 queries recur across engine arms, a query-blind probe reached
$0.48$, four times the fixed score, by memorizing per-query outcomes
rather than learning relevance; under query-disjoint folds the same
probe collapses to $0.065$ on the audited seven-arm dataset ($0.104$ on
the ten-arm extension the body reports). The leak inflated qualitative conclusions, not just
magnitudes, and Appendix~\ref{app:leakage} tabulates both regimes. A
dispersion-based per-comparison confidence flag failed held-out
confirmation twice and is released as a null
(\path{dispersion_confidence_report.json}).

\begin{figure}[t]
\centering
\includegraphics[width=0.92\linewidth]{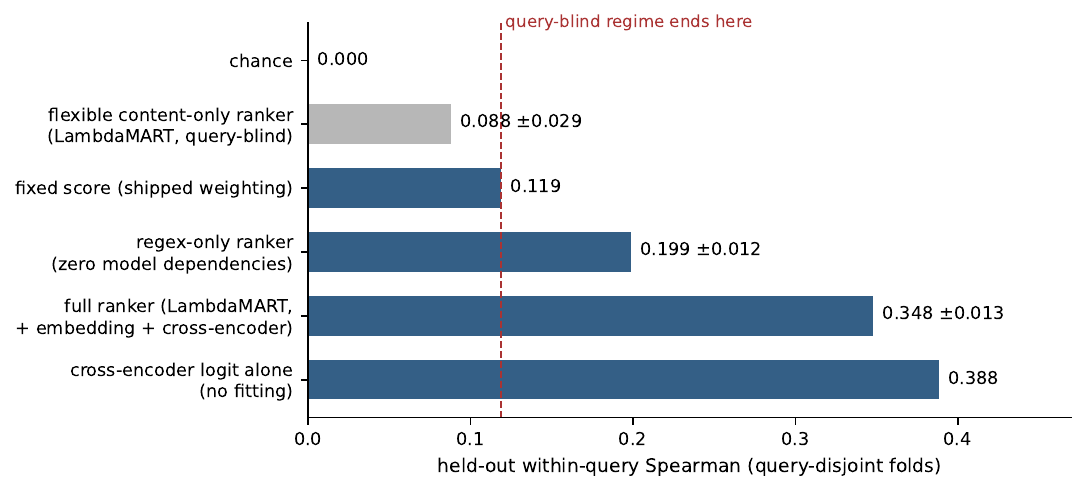}
\caption{Where the citation signal lives, all under query-disjoint
evaluation (fitted models: five-seed means; script
\texttt{geo\_paper\_fig\_signal.py}). The dashed line marks the best
query-blind value: everything to its right requires the query. The
flexible content-only ranker does not reach the fixed score, the
empirical counterpart of Proposition~\ref{prop:ceiling}.}
\label{fig:signal}
\end{figure}

\section{Threats to validity}
\label{sec:threats}

\paragraph{Anchors and calibration.} Five deterministic anchors cap the
alignment statistics (exact permutation floor $1/120$); Pearson with
bootstrap CIs is the primary metric and no ordering claim is made where
CIs overlap. Calibrating and gating on the same interventions risks
teaching to the test, mitigated by nested selection, ridge anchoring,
held-out splits, adversarial variants never seen by the fit, and a
transfer corpus; the aggressive fit is rejected precisely because it
wins alignment by sacrificing these. The strongest guard against
teaching-to-the-test is that the property we headline is not itself a
test score: the bounded-amplification result is Lemma~\ref{lem:concave}
instantiated, true for the entire concave family independently of any
gate constant, and the gaming benchmark then stresses it on attack
families the calibration never produced (Section~\ref{sec:gaming}). The residual value of anchors whose
levers measure at zero on modern engines is stated in
Section~\ref{sec:method}: the gates define the score, the anchors are
the only published causal check, and the modern-anchor measurement
closes the question empirically. Our deterministic edits approximate the
original LLM-generated ones; three published interventions requiring
rewriting are out of scope, and one of our own edits contained a bug
found and disclosed during the audit (Section~\ref{sec:ab}); three
harness bugs were found in total, and the honest inference from finding
three is a residual risk of undiscovered ones. The 163-assertion
consistency audit bounds one failure mode only, divergence between a
number printed in the paper and its released artifact; it does not and
cannot certify the semantic correctness of the harness that produced
the artifact.

\paragraph{Features and corpus.} Entropy saturates at 100 on 84\% of
base sources, so absolute levels carry strictly weaker support than
responses to edits (the renormalization trade-off and the failed
confidence flag are reported in Section~\ref{sec:outcome}); the entropy
pair double-counts one construct (a merge refit shows the redundancy is
immaterial); the three block-structural features are degenerate on
plain-text GEO-Bench and their weights unidentified; tokenization is
ASCII-only and rarity queries an English word list, so the French gate
tests the aggregation under degraded features, not language fairness.

\paragraph{Outcome scale and engines.} Ten engine families under a
controlled five-source, fixed-prompt, PAWC-share protocol are a proxy,
not open-web competition with retrieval and deduplication
\citep{puerto2025cseo}; hosted snapshots may drift (records released);
the gpt-5.x replication arms are single-run by design (no
temperature or seed exposed; power bought through scale, records
released); and the gaming benchmark's attack
surface, while extended beyond the calibration families, cannot be
exhaustive: attacks outside the released families are untested, and the
score detects none of the out-of-distribution families tested.

\paragraph{Ranking study.} 229 unique queries make the fitted models
small-sample; we claim no superiority over the best single relevance
signal; residual dataset artifacts remain possible despite the three
removed leaks; the freshness increment is observational, not causal.

\section{Reproducibility}
\label{sec:repro}

\sloppy
The reproducibility package (data, analysis code, \texttt{reproduce.sh}
and the consistency audit described below) is public at
\url{https://github.com/TW3-Partners-OS/geo-score-reproducibility}.

Two regimes apply. \emph{Exact}: every reported number reproduces
offline from released artifacts alone (per-query outcome records,
feature captures and caches, frozen weight vectors, result JSONs), with
no engine access. \emph{Fresh re-collection} from hosted engines is
possible at zero marginal cost but snapshots may drift; the released
records insulate the analyses from this. The release contains all
harness, analysis, and figure scripts named throughout the paper; the
pinned weight files, including the prior observational weighting used as
the ridge anchor (the live production weights were promoted after the
calibration runs, so every refit pins $w_0$ explicitly, a failure mode
we hit ourselves); exact hosted model identifiers with pinned decoding
(the gpt-5.x arms are single-run, disclosed in
Section~\ref{sec:outcome}); pinned library versions and fixed seeds for
every split, bootstrap, and perturbation; and a one-command
\texttt{reproduce.sh} that regenerates every offline artifact and runs a
163-assertion consistency audit checking each headline number against
its artifact. Scoring runs on commodity CPU (roughly 70~ms to 0.4~s per
page); the full strategy-study capture is about 35{,}000 scorings, a few
hours single-process, and every downstream analysis re-runs from the
released captures in minutes.

\section{Conclusion}

This paper does not propose to replace answer-side GEO evaluation; live
engines remain the causal ground truth. Its contribution is the
validation protocol, and the protocol's demonstrated outputs are the
evidence that it works: a deterministic, auditable content score whose
one distinctive property is the bounded attacker (at most $+6$ points
from amplifying any calibrated lever, decreasing with dose, sub-additive
across levers, the measured footprint of the concavity bound), all
reproducible offline at zero marginal cost. The contribution is deliberately scoped: detection of spam-style
degradation belongs to standard baselines, which the benchmark
quantifies (they dominate it, and out-of-distribution attacks evade the
score), while the score contributes the property no baseline has, the
bounded, auditable quality measure; the deployable filter layers the
two. The right deployments
follow: corpus filtering, editorial quality control, regression
testing of content pipelines, leaderboards audited under edits; not
per-query citation prediction, which the ceiling result shows no
query-agnostic score can deliver. The claims are established by
measurement standards rather than predictor standards: calibration
against causal effect sizes, falsification gates selected on training
data and confirmed held-out, per-weight identifiability in place of
per-weight optimality rhetoric, and a query-conditioned bound on what
the proxy can never resolve. Three findings deserve emphasis. First, the
2023 GEO toolbox is dead on modern engines and we measured its
replacement: the three strongest published levers moved citation on none
of ten engine families, the self-measured modern anchor vector is near
zero with one lever nominally negative, and a recalibration against it
(the released modern-anchor recalibration) strips the score of its
lever-responsive components, which settles what calibration buys today:
the gate-enforced response surface, not anchor recency. Second, no
content-only model we could fit beats the fixed score on unseen queries,
so its modest outcome correlation (Spearman $0.11$, replicated out of
family on proprietary gpt-5.x arms) reflects the ceiling that binds its
function class (Proposition~\ref{prop:ceiling}), not an engineering
shortfall, while query conditioning roughly triples the measured signal.
Third, negative results and self-audits are load-bearing: the calibrated
weighting does not beat its predecessor on outcomes, an early suggestive
intervention effect vanished when scaled, the falsification of an
alternative aggregation family rested partly on a bug in our own
negative-control edit, our first ranking evaluation leaked query
information across engine arms, and a dispersion-based confidence flag
we designed failed its held-out confirmation; all are reported and
corrected rather than absorbed. The natural next steps are anchor
measurement for levers the current literature has not tested
(structural and freshness interventions), outcome validation in real
web settings with retrieval and competition, and language-aware
tokenization for cross-lingual deployment. Beyond GEO, the protocol,
anchor to causal evidence, gate adversarially, bound with a conditioned
skyline, re-measure the anchors and report what you find, applies
wherever a deterministic proxy must stand in for an expensive generative
oracle.

\bibliographystyle{plainnat}
\bibliography{refs}

\appendix

\section{Formal specification of the eleven sub-components}
\label{app:features}

This appendix specifies each sub-component exactly as implemented in the
released scorer (\path{domain/geo/}), in the \texttt{regex\_only} code path
used throughout the paper. Notation: $x$ is the page text,
$\mathrm{clamp}(u; a, b) = \min(b, \max(a, u))$, and every feature returns a
value in $[0, 100]$ (rounding to one or two decimals as noted). Where a
required quantity is unavailable the feature returns the documented
degenerate value rather than failing; Table~\ref{tab:degenerate} collects
these. Two naming caveats are stated up front because the artifact names are
kept for compatibility with released data: \texttt{citation\_f1} is a
weighted \emph{arithmetic} mean, not a harmonic F1
(Appendix~\ref{app:citation}), and \texttt{ndcg\_score} computes the
classical DCG ratio only as a stored reference value, not as part of the
score (Appendix~\ref{app:ndcg}).

\subsection{Tokenization conventions}
\label{app:tok}

Unless stated otherwise, the information-theoretic features tokenize with
the ASCII pattern \verb|\b[a-zA-Z]+\b| on lower-cased text, keep tokens of
length $\ge 3$, and remove a fixed bilingual (EN+FR) stop list of about 180
function words. Accented characters terminate a token under this pattern;
the cross-lingual consequences are discussed in
Section~\ref{sec:threats}. Sentence segmentation, where needed, splits on
\verb|[.!?]+| and keeps segments longer than 10 characters.

\subsection{\texttt{shannon\_entropy} ($w = 0.246$)}
\label{app:shannon}

Let $V$ be the $\min(|\mathcal{V}|, 100)$ most frequent tokens after the
filtering of Appendix~\ref{app:tok}, with counts $f_1, \dots, f_V$ smoothed
by $\alpha = 10^{-10}$:
\begin{equation}
p_i = \frac{f_i + \alpha}{\sum_{j=1}^{V} (f_j + \alpha)},
\qquad
H = -\sum_{i=1}^{V} p_i \log_2 p_i .
\end{equation}
$H$ is normalized by its maximum $\log_2 V$ and shifted by a small
vocabulary bonus:
\begin{equation}
f_{\mathrm{shannon}}(x) = 10 \cdot
\mathrm{clamp}\!\left(
\min\!\left(10,\;
\frac{H}{\log_2 V} \cdot 10
+ \min\!\left(1.5,\; \tfrac{1}{3}\log_{10} \max(1, V)\right)
\right);\, 0,\, 100/10\right).
\end{equation}
Degenerate cases: fewer than 10 tokens, or an undefined $H$, return $50$;
$V < 5$ returns $30$; fewer than 10 words in the whole text returns $30$.

\paragraph{Worked example (computed by the released calculator).} A short
technical paragraph about transformer inference (three sentences, 50 words)
yields $H = 5.414$ bits over $V = 43$ vocabulary types, so
$H / \log_2 V = 0.998$ and the feature saturates at $100$. A deliberately
repetitive marketing paragraph (``Our platform helps your business grow.
Our platform gives your business tools. \dots'') yields $H = 3.384$ bits
over $V = 13$ types and scores $95.2$: normalizing by $\log_2 V$ measures
the \emph{flatness} of the frequency distribution over the surviving
vocabulary, not vocabulary size itself, so even repetitive prose scores
high in absolute terms. This ceiling behavior is quantified on GEO-Bench in
Appendix~\ref{app:diagnostics} (84\% of sources sit exactly at 100) and its
implications are discussed in Section~\ref{sec:threats}. The feature's
discriminative work in this paper happens under \emph{edits}: appending
repeated keywords concentrates the frequency distribution, drops $H$
sharply, and drives the negative response to the keyword-stuffing control.

\subsection{\texttt{information\_density} ($w = 0.244$)}
\label{app:infodensity}

A composite with internal weights $14 : 4 : 2$ (normalized):
\begin{equation}
f_{\mathrm{ID}}(x) = 0.70\, f_{\mathrm{shannon}}(x)
+ 0.20\, S_{\mathrm{lex}}(x) + 0.10\, S_{\mathrm{sem}}(x).
\end{equation}
Because its dominant term is exactly $f_{\mathrm{shannon}}$, the pair
(\texttt{shannon\_entropy}, \texttt{information\_density}) shares
$0.246 + 0.70 \times 0.244 \approx 0.417$ of the total pre-transform weight
and is correlated $0.94$ (Pearson) on GEO-Bench
(Appendix~\ref{app:diagnostics}). We keep both because they are separate
released artifacts and because the calibration operates on the aggregate,
but we disclose the double counting explicitly; a reviewer treating the
eleven weights as measuring eleven independent constructs would be misled.

$S_{\mathrm{lex}}$ (lexical diversity) combines five bounded terms,
\begin{equation}
S_{\mathrm{lex}} = 100\big(
0.35\,\min(1, \mathrm{MTLD}/120)
+ 0.20\, C_{\mathrm{Herdan}}
+ 0.15\, M
+ 0.15\, v_{\mathrm{sent}}
+ 0.15\, r_{\mathrm{reg}}\big),
\end{equation}
where $\mathrm{MTLD}$ is the bidirectional Measure of Textual Lexical
Diversity at TTR threshold $0.72$, $C_{\mathrm{Herdan}} = \ln V / \ln N$
for $N$ tokens and $V$ types, $M = \max(0, 1 - \min(1, 5 a^2))$ with
$a^2 = (\ln N - \ln V)/(\ln N)^2$ (an inverted Maas index),
$v_{\mathrm{sent}} = \mathrm{clamp}(1 - |CV - 0.5| / 0.5;\, 0, 1)$ for the
coefficient of variation $CV$ of sentence lengths, and $r_{\mathrm{reg}}$
is a register-consistency ratio over two fixed marker lists.

$S_{\mathrm{sem}}$ (vocabulary sophistication) combines mean word rarity
under \texttt{wordfreq} Zipf frequencies,
$r = \max(0, (7 - \mathrm{zipf}) / 7)$, mean syllable count, the proportion
of tokens of length $\ge 8$, alignment with RAKE key concepts extracted
from the title, and a technical-affix density; if the title yields no
concepts, $S_{\mathrm{sem}} = 50$ exactly. The Zipf lookup always queries
the English word list, a known limitation for French scoring
(Section~\ref{sec:threats}).

\subsection{\texttt{quotable\_density} ($w = 0.124$)}
\label{app:quotable}

Sentences are segmented as in Appendix~\ref{app:tok}, giving $n_s$. A
sentence is \emph{quotable} if it matches any of 13 English or 19 French
case-sensitive patterns covering (i) definitional constructions (``X is
defined as'', ``X d\'esigne'', colon definitions), (ii) enumeration leads
(``There are $n$ main\dots''), and (iii) attributed quoted speech
(a 20--200 character quotation followed or preceded by a declarative verb
and an attribution, in both quote conventions), the family added by the
French audit of Section~\ref{sec:frfix}. Matches are deduplicated and
capped at $q \le 10$; with $r = q / n_s$ and target density $10\%$,
\begin{equation}
f_{\mathrm{quot}}(x) =
\begin{cases}
100 & r \ge 0.10,\\
50 + 1000\,(r - 0.05) & 0.05 \le r < 0.10,\\
1000\, r & r < 0.05.
\end{cases}
\end{equation}
Fewer than 3 sentences returns $0$. The cap $q \le 10$ implies a
structural ceiling $r \le 10 / n_s$ for texts with more than 100
sentences: very long pages cannot reach the target density regardless of
content, a deliberate anti-stuffing trade-off that also bounds the reward
for legitimate long-form quotation.

\subsection{\texttt{entity\_density} ($w = 0.094$)}
\label{app:entity}

With $W$ the whitespace-word count (fewer than 50 words returns $0$), $E$
the number of \emph{unique} matches of case-sensitive entity patterns
(capitalized multi-word names, acronyms, dates, currency amounts, large
numbers, version strings, URLs, and French regulatory entities), $F$ the
number of factual-claim attribution matches, and $e = 100 E / W$:
\begin{equation}
b(e) =
\begin{cases}
100 & e \ge 3.5,\\
65 + 23.3\,(e - 2) & 2 \le e < 3.5,\\
35 + 30\,(e - 1) & 1 \le e < 2,\\
35\, e & e < 1,
\end{cases}
\qquad
f_{\mathrm{ent}} = \min\!\big(100,\; b(e) + \min(15, 2F) + B_{\mathrm{len}}\big),
\end{equation}
where $B_{\mathrm{len}} = \min(15, 1.5\,(E - 20))$ if $W \ge 800$ and
$E \ge 20$, else $0$.

\subsection{\texttt{semantic\_coherence} ($w = 0.084$)}
\label{app:coherence}

A composite with internal weights $1 : 4 : 12$ (normalized by 17):
\begin{equation}
f_{\mathrm{coh}}(x) = \tfrac{1}{17} T + \tfrac{4}{17} K + \tfrac{12}{17} F_{\mathrm{flow}}.
\end{equation}

\paragraph{$T$: title--content Jensen--Shannon term.} Let $P$ and $Q$ be
the smoothed unigram distributions of the title and content over their
combined top-50 vocabulary (tokenization of Appendix~\ref{app:tok}). The
implementation computes the Jensen--Shannon \emph{distance}
$d = \mathrm{JSD}^{1/2}(P \| Q)$ (base 2, via
\texttt{scipy.spatial.distance.jensenshannon}), where
\begin{equation}
\mathrm{JSD}(P \| Q) = \tfrac{1}{2} D_{\mathrm{KL}}\!\big(P \,\big\|\, \tfrac{P+Q}{2}\big)
+ \tfrac{1}{2} D_{\mathrm{KL}}\!\big(Q \,\big\|\, \tfrac{P+Q}{2}\big),
\qquad
D_{\mathrm{KL}}(A \| B) = \sum_i a_i \log_2 \frac{a_i}{b_i},
\end{equation}
and maps it through
$T_0 = \mathrm{clamp}(10 \min(10, (1 - d) \cdot 10 + \beta_V);\, 0, 100)$
with a vocabulary bonus $\beta_V = \min(2, \tfrac{1}{2}\log_{10}\max(1, V_c))$,
plus a length bonus of up to $+35$ for contents over 400 words.

\paragraph{Worked example (computed by the released calculator).} For the
title ``How transformer inference works'' against the matched technical
paragraph of Appendix~\ref{app:shannon}, $d = 0.969$ and $T_0 = 11.4$; against
a deliberately unrelated real-estate paragraph, $d = 1.0$ and $T_0 = 7.1$.
Because a 4-token title and a 40-token paragraph have nearly disjoint
unigram mass, $d$ is close to 1 even for on-topic pairs: the term
discriminates weakly and carries only $1/17$ of the composite, which is
itself why the composite's empirical variance is dominated by
$F_{\mathrm{flow}}$. We report this rather than describe the
feature as a semantic-similarity measure.

\paragraph{$K$: keyword alignment.} RAKE concepts from the title are
matched against the content by a weighted combination of TF-IDF cosine
(scikit-learn \texttt{TfidfVectorizer}, $n$-grams $(1,2)$), concept
coverage, occurrence density with optimum $3\%$, and positional proximity,
plus a length bonus; no title concepts returns $K = 50$.

\paragraph{$F_{\mathrm{flow}}$: discourse flow.} Five bounded terms
(adjacent-sentence lexical overlap, transition-marker density with optimum
$0.3$ per sentence, entity continuity over a 3-sentence window, paragraph
structure, sentence-length variety). One implementation detail matters for
interpretation: the semantic agent receives the markdown flattened to a
single line, so the paragraph term is the constant $0.4$ for any normal
article ($0.7$ for very short ones). This constant offset does not affect
rankings and is invisible to the intervention deltas (it cancels in
Eq.~\eqref{eq:delta}), but absolute values of $f_{\mathrm{coh}}$ should not
be read as a calibrated coherence measurement; the empirical distribution
(mean $52.0$, sd $4.3$) reflects it.

\subsection{\texttt{self\_containment} ($w = 0.073$)}
\label{app:selfcont}

Paragraphs (split on blank lines, $\ge 30$ characters, excluding headings,
code blocks, and list items) are scored by their opening: $20$ if the first
sentence matches a context-dependent opener (pronoun leads such as ``It
is'', ``This approach'', ``Cela permet''), else $100$ if it matches an
explicit-subject pattern (capitalized subject with verb, leading acronym,
or leading statistic), else $40$ if shorter than 20 words, $65$ if longer
than 200 words, else $82$. The feature is the mean over paragraphs; no
analyzable paragraph returns $20$. On GEO-Bench, sources are mostly single
plain-text passages, so the feature takes only 4 distinct values
(Appendix~\ref{app:diagnostics}).

\subsection{\texttt{statistic\_density} ($w = 0.051$)}
\label{app:statdensity}

$W$ as above (fewer than 50 words returns $0$); $S$ counts unique matches
of 9 English or 20 French numeric-evidence patterns (percentages, currency
with magnitude words, multiplicative comparisons, ratios such as ``$a$ out
of $b$'', hedged quantities); $s = 100 S / W$:
\begin{equation}
b(s) =
\begin{cases}
100 & s \ge 1.5,\\
50 + 71.4\,(s - 0.8) & 0.8 \le s < 1.5,\\
62.5\, s & s < 0.8,
\end{cases}
\qquad
f_{\mathrm{stat}} = \min\!\big(100,\; b(s) + \min(25, 10 D)\big),
\end{equation}
where $D$ counts dated statistics (a number and a year in the same
sentence or parenthesis).

\subsection{\texttt{mmr\_score} ($w = 0.036$)}
\label{app:mmr}

The text is segmented into blocks at headings and blank lines; blocks with
$\ge 30$ words are \emph{valid}. With fewer than 2 valid blocks the
composite returns the constant $20$, which is the dominant case on
GEO-Bench (plain-text sources without markdown structure; see
Appendix~\ref{app:diagnostics}). Otherwise
$f_{\mathrm{mmr}} = 0.25 L + 0.65 R + 0.10 Q$, where $L$ is a TF-IDF
lexical-diversity term, $R$ is the redundancy score of
Appendix~\ref{app:redundancy}, and $Q$ scores a greedy Maximal Marginal
Relevance selection \citep{carbonell1998mmr} with $\lambda = 0.7$:
\begin{equation}
\mathrm{MMR}(S_i) = \lambda\, \mathrm{sim}(S_i, Q_c)
- (1 - \lambda) \max_{S_j \in \mathcal{S}} \mathrm{sim}(S_i, S_j),
\end{equation}
over TF-IDF similarities, where $Q_c$ is a pseudo-query of the document's
top-8 TF-IDF concepts, scored by selection ratio, mean MMR value, and
selected-count bonus.

\subsection{\texttt{citation\_f1} ($w = 0.019$)}
\label{app:citation}

Citations are unique matches (capped at 20) of URL and attribution
patterns (``according to X'', ``(Source, 2024)'', ``selon X'', ``\'etude
de X''), each with a fixed prior confidence $c_j \in \{0.6, 0.7\}$ by
source type. With $M$ factual-claim matches and $C$ the citation set:
\begin{equation}
P = 100 \cdot \frac{|\{j : c_j \ge 0.55\}|}{|C|},
\qquad
R = \min\!\Big(100,\; 100 \cdot \frac{|C|}{\max(1, M/5)}\Big),
\qquad
Q = 100\,\bar{c},
\end{equation}
\begin{equation}
f_{\mathrm{cit}} = 0.50\, P + 0.30\, R + 0.20\, Q .
\end{equation}
Despite its artifact name this is a weighted arithmetic mean, not a
harmonic F1; and since all prior confidences exceed the threshold $0.55$,
$P = 100$ whenever at least one citation exists, so in the
\texttt{regex\_only} path the feature is effectively a saturating citation
count: $0$ without citations, $\ge 62$ with one. (In the full production
path, search-validated confidences can move $P$; that machinery is
disabled here and reported for completeness only.)

\subsection{\texttt{ndcg\_score} ($w = 0.015$)}
\label{app:ndcg}

Markdown sections are parsed at headings (fewer than 2 sections returns
the constant $20$; at most 10 are analyzed). The score is
$0.10\, S_{\mathrm{rel}} + 0.55\, S_{\mathrm{qual}} + 0.20\, S_{\mathrm{hier}} + 0.15\, S_{\mathrm{pos}}$:
BM25 relevance of each section to the document's top-8 TF-IDF concepts
($k_1 = 1.5$, $b = 0.75$, max-normalized), a per-section quality term
(length tier, sentence count, title informativeness, TTR), a
heading-hierarchy consistency term penalizing level jumps, and a position
term rewarding front-loaded sections. The classical
$\mathrm{DCG}/\mathrm{IDCG}$ value \citep{jarvelin2002ndcg} is computed and
stored as a reference metadata field but does not enter the score; the
artifact name is kept for data compatibility and the paper refers to the
feature as a heading-structure quality composite.

\subsection{\texttt{semantic\_redundancy} ($w = 0.015$)}
\label{app:redundancy}

On the valid blocks of Appendix~\ref{app:mmr} (fewer than 2 blocks:
constant $20$ via the shared gate), the feature scores \emph{diversity}
(high = low redundancy):
\begin{equation}
f_{\mathrm{red}} = 100\big(
0.40\,(1 - \overline{\cos})
+ 0.25\,(1 - \bar{J})
+ 0.20\,\gamma
+ 0.15\,(1 - \bar{V})\big),
\end{equation}
with $\overline{\cos}$ the mean pairwise TF-IDF cosine between blocks,
$\bar{J}$ the mean pairwise token Jaccard, $\bar{V}$ the mean pairwise
vocabulary Jaccard, and $\gamma$ a gzip compression-ratio diversity,
$\gamma = \mathrm{clamp}\big((|\mathrm{gzip}(x)| / |x| - 0.1)/0.9;\, 0, 1\big)$.

\subsection{Aggregation, defaults, and caps}
\label{app:aggregation}

Missing sub-component values default to $50$ before the transform. The
content score applies $g(u) = 100\sqrt{u / 100}$ element-wise
(Eq.~\eqref{eq:content}), blends freshness with weight $0.08$
(Eq.~\eqref{eq:blend}), and applies word-count caps to the final value: a
page under 100 words is capped at 35 points, under 200 at 50, under 300 at
65. (An earlier draft described ``listing caps''; the released scorer
implements length caps only, and a declared diversity cap is dead code.)
Upstream degenerate-input gates precede the sub-components: near-empty text
scores 5, detected gibberish 10 or 5, and function-word-free word salad 10.

The freshness component $F(x) \in [0, 100]$ of Eq.~\eqref{eq:blend},
carrying $8\%$ of the final score, is itself a deterministic weighted mean
of four bounded signals extracted by regex from the text:
\begin{equation}
F = 0.40\, F_{\mathrm{date}} + 0.20\, F_{\mathrm{temporal}}
+ 0.25\, F_{\mathrm{refs}} + 0.15\, F_{\mathrm{updates}},
\end{equation}
where $F_{\mathrm{date}}$ scores the recency of detected publication or
modification dates, $F_{\mathrm{temporal}}$ the density of
recency-indicating language, $F_{\mathrm{refs}}$ the recency of years
appearing in references and citations, and $F_{\mathrm{updates}}$ explicit
update signals; absent evidence, $F = 50$. $F_{\mathrm{refs}}$ is the term
responsible for the French cite-sources residual discussed in
Section~\ref{sec:ab}: a cite block containing recent years raises $F$.

\paragraph{Corpus provenance.} The 500 English sources are the first 500
entries of the released \path{geo_bench_sources_1000.jsonl}, built by
streaming the public GEO-Bench \texttt{test} configuration in dataset
order and keeping, per query, the first source with at least 200 words,
truncated to 400 words (one source per query for topical diversity). The
60 French sources, released as \path{geo_bench_sources_fr.jsonl}, are
French Wikipedia extracts from a fixed title list (250--450 words each)
plus a small local validation set, fetched by the harness's
\texttt{fetch-fr} command; French is a transfer \emph{check}, not a claim
of representative French web coverage.

\begin{table}[t]
\centering
\footnotesize
\begin{tabular}{lll}
\toprule
Sub-component & Degenerate condition & Value \\
\midrule
shannon\_entropy & $< 10$ tokens / undefined $H$; $V < 5$; $< 10$ words & 50; 30; 30 \\
information\_density & $< 10$ words & 6 \\
quotable\_density & $< 3$ sentences & 0 \\
entity\_density & $< 50$ words & 0 \\
semantic\_coherence & $< 10$ words & 15 \\
self\_containment & no analyzable paragraph & 20 \\
statistic\_density & $< 50$ words & 0 \\
mmr\_score & $< 2$ valid blocks & 20 \\
citation\_f1 & no citation match & 0 \\
ndcg\_score & $< 2$ heading sections & 20 \\
semantic\_redundancy & $< 2$ valid blocks & 20 \\
any & field missing from agent output & 50 \\
\bottomrule
\end{tabular}
\caption{Degenerate-input values of the eleven sub-components. These
constants explain the discrete empirical distributions of the structural
features on GEO-Bench (Appendix~\ref{app:diagnostics}): plain-text sources
without markdown headings hit the ``$< 2$ blocks/sections'' paths, which is
also the mechanism behind the parse-failure proxy leak removed in
Section~\ref{sec:ltr}.}
\label{tab:degenerate}
\end{table}

\section{Robustness and production replication: full detail}
\label{app:robfull}
\subsection*{Robustness of the selection}
\label{sec:robust}

The selection is a plateau, not a knife-edge
(\path{ab_robustness_report_ksfix.json}). (i) \emph{Split stability}:
across 10 seeded train/test re-splits with full nested refit, held-out
Pearson spans $[0.934, 0.972]$ (mean $0.951$) and the keyword-stuffing
control is negative on all 10. (ii) \emph{Weight sensitivity}: under 200
random $\pm 10\%$ perturbations of the selected weights, the worst held-out
alignment is $0.929$ and the control's worst case is $-9.16$.
(iii) \emph{Regularization}: alignment varies only from $0.927$ to $0.949$
across the entire $\lambda$ grid. (iv) \emph{Transform}: refitting under
$g(u) = 100\,(u/100)^{q}$ for $q \in [0.4, 1.0]$ leaves alignment nearly
flat ($0.942$ to $0.956$) with the stuffing control negative throughout
($-8.9$ to $-12.6$). Under the corrected control the penalty magnitude in
fact \emph{increases} toward the raw end of the sweep (the pre-fix capture
showed the opposite ordering, an artifact of the letter-noise edit, and we
correct the record here): magnitude of the negative control is therefore
\emph{not} the argument for concavity. The argument is the two gates the
raw end fails and the sqrt point passes: dose-response under repeated
proven levers (worst step $-0.74$ raw versus $-0.36$ sqrt,
Table~\ref{tab:ab}) and gaming saturation. Concavity is selected for how
it treats legitimate repeated improvement and adversarial amplification,
at essentially no alignment cost within the flat region.

\subsection*{End-to-end check in the production pipeline}

Swapping the selected weights into the full production scorer (caps and
freshness included) moves quotation from last to first among levers in
English ($\Delta = 0.68 \to 5.44$; rank correlation with the anchors
$-1.0 \to +0.5$). With the corrected stuffing edit, the full-pipeline
negative control holds at both doses under both weightings
(shipped: $-4.0$ at dose 1 and $-11.2$ at dose 8 in English, $-2.1$/$-8.2$ in
French; prior: $-5.5$/$-18.3$ and $-3.3$/$-15.5$; artifacts
\path{realpipeline_ks_*.json}). In French, the quotation response increases
$4.6\times$ but cite-sources still ranks first in the full pipeline; the
residual is consistent with the freshness component rewarding the recent
years contained in the cite block, and is reported as a limitation rather
than hidden by the offline numbers.

\section{Design rationale for the eleven sub-components}
\label{app:designfull}

\label{sec:design}

Every sub-component had to pass four admission criteria, stated here so
that the inclusion of each feature, and the exclusion of the obvious
alternatives, is checkable rather than asserted:

\begin{enumerate}
\item \emph{Deterministic and dependency-free} (Problem~\ref{prob:score}
  (iii)): closed-form regex/statistical computation, no learned model at
  scoring time. This alone excludes the strongest known quality proxies,
  LM perplexity \citep{wenzek2020ccnet} and embedding similarity: both
  would make the leaderboard score depend on a model artifact and its
  version, forfeiting the auditability claim. (The ranking study of
  Section~\ref{sec:ltr} re-admits them as \emph{optional} features
  precisely to measure what excluding them costs.)
\item \emph{Construct relevance to citability}: a documented reason, prior
  to our calibration, to believe the construct affects whether generative
  engines lift and cite a span. Table~\ref{tab:rationale} gives the
  per-feature rationale with its source.
\item \emph{A causal or robustness role}: after calibration, the feature
  must either carry an anchor (Table~\ref{tab:featanchor}), carry a gate
  (the entropy mass detects stuffing), or cover a construct the anchors
  cannot probe, in which case it is labeled as prior-inherited
  (Section~\ref{sec:ident}) rather than dressed in causal language.
\item \emph{Language portability}: computable in English and French with
  pattern swaps only (imperfectly achieved; the ASCII-tokenization and
  English-Zipf limitations are in Section~\ref{sec:threats}).
\end{enumerate}

\begin{table}[t]
\centering
\footnotesize
\setlength{\tabcolsep}{4pt}
\begin{tabular}{lp{5.2cm}p{5.6cm}}
\toprule
Sub-component & Construct (why this signal) & Why it stays (evidence class) \\
\midrule
shannon\_entropy & Well-formed informative prose keeps a flat token
distribution; entropy-rate regularity \citep{shannon1948,
genzel2002entropy} & \emph{Gate carrier}: the score's principal stuffing
detector ($-7.4$ under the negative control). Limit: saturates at 100 on
84\% of base sources (App.~\ref{app:diagnostics}) \\
information\_density & Uniform Information Density: effective text spreads
information evenly \citep{jaeger2010redundancy, meister2021uid} &
\emph{Gate carrier} ($-9.1$ under control). Limit: 70\% overlap with
shannon\_entropy, one construct under two names \\
quotable\_density & Engines lift definitional and attributed spans; the
literal realization of Quotation Addition \citep{aggarwal2024geo} &
\emph{Anchor carrier}, sharply identified: LOFO collapses alignment
$0.949 \to 0.688$, admissible $[0.12, 0.16]$. Limit: match cap $q \le 10$
penalizes long pages (App.~\ref{app:quotable}) \\
entity\_density & Named, checkable referents are what answers attribute
\citep{liu2023verifiability} & \emph{Weakly identified}: secondary
responder to three anchors, weight otherwise prior-inherited. Limit: 57\%
of sources at ceiling \\
semantic\_coherence & Local discourse coherence, entity-grid style
\citep{barzilay2008entity}; topical unity of the page &
\emph{Prior-retained}, bounded above by the profile analysis
($w \le 0.12$). Limit: JS term discriminates weakly, flow term dominated by
a constant offset (App.~\ref{app:coherence}) \\
self\_containment & A span must stand alone to survive extraction into an
answer; context-dependent openers break attribution
\citep{gao2023alce} & \emph{Prior-retained}: anchors do not probe it.
Limit: 4 distinct values on plain-text sources \\
statistic\_density & Realization of Statistics Addition
\citep{aggarwal2024geo}; numeric evidence invites citation & \emph{Anchor
+ gate carrier}: identified from both sides ($w \in [0.02, 0.05]$). Limit:
zero on 61\% of base sources \\
mmr\_score & Coverage/diversity of sections \citep{carbonell1998mmr} &
\emph{Structural candidate}: prior-retained, degenerate on this corpus,
kept for structured pages \\
citation\_f1 & Realization of Cite Sources \citep{aggarwal2024geo};
sourced claims are more attributable & \emph{Anchor carrier}, capped above
($w \le 0.02$) because its carrier gain is enormous ($+82$ points). Limit:
effectively a saturating citation count in regex-only mode
(App.~\ref{app:citation}) \\
ndcg\_score & Heading hierarchy and front-loading of content
\citep{jarvelin2002ndcg} & \emph{Structural candidate}: prior-retained,
degenerate on this corpus \\
semantic\_redundancy & Near-duplicate blocks dilute information; redundancy
penalized since \citet{carbonell1998mmr} & \emph{Structural candidate}:
prior-retained, degenerate on this corpus \\
\bottomrule
\end{tabular}
\caption{Per-feature design rationale. ``Why it stays'' classifies the
post-calibration evidence: anchor carrier and gate carrier are
empirically identified roles (Sections~\ref{sec:method}
and~\ref{sec:ident}); prior-inherited means the anchors neither require
nor forbid the feature and its weight comes from the ridge anchor $w_0$,
stated as such.}
\label{tab:rationale}
\end{table}

Candidates that fail a criterion are excluded with the failing criterion
named and a status, \emph{permanent} (the criterion is constitutive of the
artifact) or \emph{deferred to a future anchor set} (excluded on present evidence,
revisitable with better anchors): LM perplexity and embedding coherence
fail (1), permanent for the leaderboard score by construction (the ranking
study re-admits them as optional features); Flesch-style readability fails
(3) and largely duplicates the lexical terms already inside
\texttt{information\_density}, deferred; type-token ratio is already a
component of the same composite, permanent as a separate feature;
sentiment and subjectivity lexicons fail (2), no documented citability
link, permanent absent new evidence; page-authority signals (backlinks,
domain age) fail (1) and (4), requiring external data, permanent for the
text-only score; a dedicated technical-terminology density passes (1),
(2), (4) but is deferred to a future anchor set because its only current justification
would be the weakest anchor of an underdetermined system
(Section~\ref{sec:ident}). Freshness signals pass the criteria but are
kept \emph{outside} the content score in a separate blended component
(Eq.~\ref{eq:blend}) because recency is a property of the page's
maintenance, not of its prose; an observational ablation finds the blend
carries a small positive increment on citation outcomes
(Section~\ref{sec:outcome}), and its causal validation is future work.

\begin{table}[t]
\centering
\small
\begin{tabular}{llp{7.6cm}}
\toprule
Sub-component & Weight $w_i$ & Definition (sketch) \\
\midrule
shannon\_entropy & 0.246 & Token-distribution entropy~\citep{shannon1948}, normalized \\
information\_density & 0.244 & Composite density index (Shannon, lexical, semantic, weighted) \\
quotable\_density & 0.124 & Density of definitional and attributed-quote sentences (Section~\ref{sec:frfix}) \\
entity\_density & 0.094 & Named-entity-like token density (capitalization and lexicon heuristics) \\
semantic\_coherence & 0.084 & Discourse-flow composite (adjacent-sentence overlap, transition markers, entity continuity); minor title-content Jensen-Shannon term (App.~\ref{app:coherence}) \\
self\_containment & 0.073 & Stand-alone quality of paragraphs (context-dependence penalties) \\
statistic\_density & 0.051 & Density of numerals, units, percentages \\
mmr\_score & 0.036 & Block-diversity composite with an MMR-selected
quality term~\citep{carbonell1998mmr} \\
citation\_f1 & 0.019 & Outbound-citation presence and recall composite
(arithmetic, not harmonic; Appendix~\ref{app:citation}) \\
ndcg\_score & 0.015 & Heading-structure quality composite (DCG value
stored as metadata only; Appendix~\ref{app:ndcg}) \\
semantic\_redundancy & 0.015 & Inter-block diversity (inverse redundancy:
TF-IDF cosine, Jaccard, gzip ratio) \\
\bottomrule
\end{tabular}
\caption{The eleven sub-components and the released weights, with
corrected definition sketches (full formulas in
Appendix~\ref{app:features}). Weights are intervention-calibrated
(Section~\ref{sec:method}); no feature is zeroed. The largest change versus
the prior production weighting is quotable\_density, $0 \to 0.124$: the
score was previously blind to quotation, the strongest published visibility
lever ($+42.6$\%). Provenance: this vector was fitted before the
quotable-density coverage fix of Section~\ref{sec:frfix}; re-evaluated on
the corrected capture it passes all five gates with held-out anchor Pearson
$0.970$ (Section~\ref{sec:ident}), and re-fitting on the corrected capture
moves every weight by less than $0.04$ except entity\_density
($0.094 \to 0.156$), all within the admissible plateau of
Table~\ref{tab:ident}.}
\label{tab:features}
\end{table}

\section{Weight identifiability: full analysis}
\label{app:identfull}

A reviewer should ask of Table~\ref{tab:features}: why $0.246$ for entropy
and not $0.15$? Why is quotation's carrier at $0.124$ and citation's at
$0.019$ when their anchors are $+42.6\%$ and $+27.7\%$? With five anchors
and eleven weights, no per-weight optimality proof is possible, and the
underdetermination is structural rather than a sample-size accident.

\begin{proposition}[The alignment optimum is a plateau]
\label{prop:plateau}
Let $M \in \mathbb{R}^{5 \times d}$ be the response matrix of
Eq.~\eqref{eq:fit} and $r = \operatorname{rank}(M) \le 5$. If $w^\star$
maximizes $w \mapsto \cos(M w, \gamma)$ over the simplex $\Delta^{d-1}$,
then every $w \in \Delta^{d-1}$ with $w - w^\star \in \ker M$ attains the
same objective value, and $\dim \ker M \ge d - 5 = 6$. Alignment alone
therefore cannot identify the weight vector beyond (at least) a
six-dimensional set; any per-weight statement requires the gates or the
prior.
\end{proposition}

\begin{proof}
If $w - w^\star \in \ker M$ then $M w = M w^\star$, so the cosine is
unchanged; $\dim \ker M = d - r \ge d - 5$.
\end{proof}

{\sloppy What the data \emph{can}
support is a characterization of that plateau, which we give in three parts
(artifact \path{identifiability_report.json}, script
\path{geo_paper_identifiability.py}). First, \emph{necessity}:
leave-one-feature-out (LOFO) refits, fixing $w_i = 0$ and refitting the
rest. Second, \emph{admissible intervals}: profile refits fixing $w_i = c$
on the grid $c \in \{0, 0.02, 0.05, 0.08, 0.12, 0.16, 0.20, 0.25,
0.30, 0.35, 0.40, 0.50\}$ and refitting the remaining weights on the
training split; $c$ is admissible when the held-out anchor Pearson stays
at or above $0.94$ (the split-stability floor of
Section~\ref{sec:robust}) \emph{and} all five gates pass. Third, the
$\lambda \to 0$ limit of Section~\ref{sec:method}, which shows what
alignment alone would do with the freedom.\par}

% (Table ref{tab:ident} appears in the body summary.)

Table~\ref{tab:ident} and Figure~\ref{fig:profiles} show the plateau has
structure: the weights split into three regimes, and each weight's value
is justified by a different kind of evidence. The point of
Figure~\ref{fig:profiles} is visual and immediate: for a sharply
identified weight the admissible region is a narrow window around the
released value; for a weakly identified one the curve is flat and
everything is admissible, which is precisely the underdetermination that
Problem~\ref{prob:score} announced.

% (Figure ref{fig:profiles} appears in the body summary.)

\paragraph{Sharply identified carriers.} \texttt{quotable\_density} is the
only feature whose removal collapses alignment ($0.949 \to 0.688$): it is
the sole carrier of the strongest anchor (Table~\ref{tab:featanchor}), and
its admissible interval $[0.12, 0.16]$ is the narrowest of any feature.
Its released weight $0.124$ is not a tuned preference but the value the
anchor data actually pins down. \texttt{statistic\_density} is identified
from both sides: removing it costs alignment ($0.921$) \emph{and} breaks
the dose-response gate, while raising it past $\approx 0.05$ over-rewards
repeated statistics dosing; its small weight despite a $+32.8\%$ anchor is
forced by the gate, not chosen. Its released value $0.051$ sits exactly at
the upper admissible grid point (within grid resolution), and the released
vector as a whole passes all gates jointly, which is the statement that
matters: the per-coordinate intervals are slices, not a product box. \texttt{citation\_f1} is capped above at
$0.02$: its response to its own anchor is enormous ($+82$ points,
Table~\ref{tab:featanchor}), so any larger weight makes cite-sources
dominate the intervention ordering and misalign the ensemble. The
asymmetry a reviewer might flag (the $+42.6\%$ anchor gets $w = 0.124$, the
$+27.7\%$ anchor $w = 0.019$) is thus explained by carrier gain: weights
compensate for how loudly each feature responds to its lever, because what
is calibrated is the product of weight and feature response, not the weight
alone.

\paragraph{Control mass.} The entropy family carries little anchor
alignment (LOFO $0.939$ and $0.930$ with gates intact) but is the negative
control's detector: it produces the only large negative responses to
keyword stuffing (Table~\ref{tab:featanchor}), and the $\lambda \to 0$ fit
below shows what happens without it. Its exact split across the two
correlated artifact names is immaterial, and a dedicated merge refit
closes the question: constraining \texttt{shannon\_entropy} to exactly
zero and refitting under the identical protocol yields held-out Pearson
$0.944$ with 5/5 gates on both splits and French transfer $0.824$,
against $0.949$, 5/5, and $0.836$ for the reference fit; the symmetric
merge (all entropy mass on \texttt{shannon\_entropy}) scores $0.935$,
also 5/5 (\path{entropy_merge_report.json}). Carrying the entropy mass
under one artifact name instead of two costs $0.005$ alignment and
$0.012$ French transfer, with every gate margin still comfortable: the
split is a released-data convention, not a load-bearing
modeling choice. What the data supports is that a large entropy mass must
exist, with \texttt{information\_density} bounded below at $0.12$ on the
profile grid.

\paragraph{Weakly identified remainder.}{\sloppy{} Five features
(\texttt{self\_containment}, \texttt{mmr}, \texttt{ndcg},
\texttt{semantic\_redundancy}, and largely \texttt{entity\_density}) admit
$[0, 0.50]$: the anchor-plus-gates evidence neither requires nor forbids
them at any tested weight. Their released weights are inherited from the
ridge anchor $w_0$, i.e.\ from the prior production weighting, and we say
so plainly rather than dress them in causal language. They are retained
because zeroing them is not required either, because they cover constructs
the anchors do not probe (structure, redundancy), and because the ranking
study of Section~\ref{sec:ltr} independently flags their degenerate
behavior. A future anchor set with structural interventions would identify
them or zero them; the present one cannot, and Table~\ref{tab:ident}
states that limit.\par}

\paragraph{What alignment alone would do.} The $\lambda \to 0$ fit zeroes
seven of the eleven features, including \texttt{shannon\_entropy} and
\texttt{citation\_f1}, and concentrates on four
(entity $0.62$, quotable $0.19$, information\_density $0.11$,
statistic $0.08$), improving training alignment (cosine $0.957$ versus
$0.929$) but generalizing worse (held-out Pearson $0.920$) with a negative
control three times weaker ($-2.96$ versus $-9.65$ at dose 8). Alignment
alone, given eleven degrees of freedom and five targets, buys training fit
by discarding most of the mass that detects stuffing. The ridge anchor and
the gates are what prevent this, which is the constrained-optimization
framing of Eq.~\eqref{eq:constrained} in action.

\paragraph{Rejected candidate features.} Section~\ref{sec:design} lists the
excluded candidates against the admission criteria; the identifiability
analysis adds the quantitative case for two of those exclusions. A
dedicated technical-terminology density (the natural carrier for the
under-realized Technical Terms anchor, Table~\ref{tab:featanchor}) is the
obvious next feature; we exclude it here because adding a feature to chase
the weakest anchor of an already underdetermined system invites teaching to
the test, and flag it for a future anchor set. Readability formulas would
add mass to the weakly identified regime, the part of the weight vector the
data already cannot pin down. Naturalness (AI-detection) and domain trust
exist in the production system and are deliberately kept \emph{out} of the
score (Section~\ref{sec:algo}); naturalness in particular is
anti-correlated with citation in our production data and mixing it into a
citability score would conflate two constructs
\citep{jacobs2021measurement}.

\section{The query-conditioned citation predictor: full study}
\label{app:ltrfull}

This section is the measuring instrument for Proposition~\ref{prop:ceiling}:
it estimates where the content-only ceiling sits and how much of the
remaining citation signal query conditioning recovers. Citation is
relational: for a given query, an engine chooses among candidate
sources, so predicting citation is at heart a \emph{query--source ranking}
problem rather than a document-scoring problem, and the query-agnostic
score is not meant to solve it. We build within-query ranking models on
the collected citation outcomes: the same deterministic feature stack,
extended with query--source relevance features, evaluated with
cross-validation whose folds are drawn over \emph{unique queries}, so that
no query is shared between train and test in any form; the paragraph
below explains why this detail is load-bearing.

\paragraph{Features and models.} Table~\ref{tab:featcat} organizes the
feature stack. The two model-based categories are explicitly \emph{optional}:
they are open-weights, deterministic, cached, and free of marginal API cost,
and the dependency ablation below quantifies exactly what each buys. We
compare the fixed score, ridge regression, pairwise learning-to-rank,
gradient-boosted trees (GBM), and LightGBM LambdaMART
(lambdarank)~\citep{wu2010lambdamart}, reporting held-out within-query
Spearman, hit@1, and the MRR of the top-cited source. All LambdaMART
results use the harness-default configuration (200 trees, 15 leaves,
learning rate $0.03$), fixed before any evaluation; a capacity grid is
reported as a sensitivity check, not used for selection.

\paragraph{A leakage audit: fold construction is load-bearing.} The first
version of this study drew CV folds over (engine, query) \emph{records}.
The 819 ranking groups of the original seven-arm study (1{,}269 after the
July-2026 gpt-5.x extension; both cover the same 229 unique queries) span
few distinct queries, and 150 queries
appear under two or more engines, so the same query could sit in train
under one engine and in test under another; because engines' within-query
citation choices correlate, a flexible model can then memorize per-query
source outcomes instead of learning transferable relevance. The probe
that exposed the leak is the ceiling experiment itself: a LambdaMART
restricted to \emph{query-blind} content features reached $\rho = 0.48$
under record-level folds, four times the fixed score, which
Proposition~\ref{prop:ceiling} rules out as genuine generalization; the
model was fingerprinting sources it had already seen. Under query-disjoint
folds the same probe collapses to $0.065$ ($0.077$ across seeds; artifact
\path{ltr_strict_report_7eng.json}), consistent with the value in
Table~\ref{tab:leak}. Every number below uses
query-disjoint folds, and Appendix~\ref{app:leakage} tabulates the
record-fold and query-fold results side by side: the leak inflated not
just magnitudes but qualitative conclusions. Two further leaks identified
earlier remain excluded from all reported numbers: the three degenerate
structural sub-components act as a parse-failure proxy
(Appendix~\ref{app:diagnostics}), and the source's prompt position is a
presentation artifact that alone ranks at $\rho = 0.165$
(Table~\ref{tab:ltr}). One retained feature deserves disambiguation
because its name invites the same suspicion: \texttt{q\_first\_pos\_frac}
is the position of the earliest query-term match \emph{inside the source
text} (a lead-answering signal), not the source's position in the prompt.

\begin{table}[t]
\centering
\footnotesize
\setlength{\tabcolsep}{4pt}
\begin{tabular}{p{3.0cm}p{6.0cm}p{3.6cm}c}
\toprule
Category & Examples & Dependency & Regex-only \\
\midrule
Content quality (8 of 11; the degenerate structural trio is excluded,
see text) & entropy, information density, quotable density,
entity density, statistic density & none (regex/statistical) & yes \\
Lexical query relevance & BM25 over the candidate set, TF-IDF cosine,
term/entity overlap, longest n-gram, proximity, lead overlap, char-bigram
Dice & none (regex/statistical) & yes \\
Answer structure & wh-type matched to answer-type presence (year, number,
proper noun, place), definitional lead & none (regex) & yes \\
Semantic embedding (opt.) & query--source cosine, margin, within-query rank,
inter-source redundancy & \texttt{bge-m3}, open weights, no marginal API
cost & no \\
Cross-encoder (opt.) & relevance logit, margin, within-query rank &
\texttt{rerank-qa-mistral-4b}, open weights, no marginal API cost & no \\
\bottomrule
\end{tabular}
\caption{Ranker feature stack by category. ``Regex-only'' marks the
categories used in the zero-dependency ablation of Table~\ref{tab:ltr}.}
\label{tab:featcat}
\end{table}

\begin{table}[t]
\centering
\footnotesize
\begin{tabular}{lccc}
\toprule
Method (query-disjoint 5-fold CV; 229 queries, $n = 1{,}269$ groups, 10
engines) & within-$q$ $\rho$ & hit@1 & MRR \\
\midrule
\multicolumn{4}{l}{\emph{Baselines, no fitting} (\path{ltr_baselines_report.json})} \\
random & $0.00$ & $0.20$ & --- \\
source position (earlier-listed first; excluded leak) & $0.131$ & $0.33$ & --- \\
BM25 only & $0.147$ & $0.28$ & --- \\
the score (query-agnostic, fixed) & $0.119^{\dagger}$ & $0.25$ & --- \\
embedding cosine only & $0.286$ & $0.36$ & --- \\
cross-encoder logit only & $\mathbf{0.388}$ & $\mathbf{0.45}$ & --- \\
\midrule
\multicolumn{4}{l}{\emph{Fitted models} (\path{ltr_strict_report.json})} \\
ridge, content features only & $0.075$ & $0.26$ & $0.50$ \\
ridge, query features only & $0.370$ & $0.46$ & $0.66$ \\
ridge, content $+$ query & $0.377$ & $0.44$ & $0.65$ \\
GBM, content $+$ query & $0.372$ & $0.45$ & $0.65$ \\
\midrule
\multicolumn{4}{l}{\emph{LambdaMART dependency ablation}} \\
content-only features (query-blind ceiling probe) & $0.104$ & $0.24$ & $0.50$ \\
regex-only features (no model dependencies) & $0.222$ & $0.34$ & $0.56$ \\
$+$ embedding & $0.286$ & $0.33$ & $0.57$ \\
$+$ cross-encoder & $0.357$ & $0.40$ & $0.62$ \\
full & $0.368$ & $0.41$ & $0.62$ \\
\bottomrule
\end{tabular}
\caption{Baselines, fitted models, and dependency ablation under the
query-disjoint protocol, over all ten engine arms (the seven original
plus the gpt-5.x replication). Five-seed fold reshuffles give
$0.348 \pm 0.013$ (full), $0.199 \pm 0.012$ (regex-only), and
$0.088 \pm 0.029$ (content-only). The full LambdaMART improves on the
fixed score by $+0.248$ per query (paired $t = 13.2$,
$p < 10^{-36}$, Wilcoxon $p < 10^{-34}$, $n = 1{,}213$). The no-fit
baselines involve no training and are unaffected by fold design; chance
hit@1 is $0.20$. The query-blind content-only probe does not beat the
fixed score, and the strongest single signal, the cross-encoder logit,
is not reliably beaten by any fitted model at this number of unique
queries.
$^{\dagger}$The fixed-score row shows the LTR harness value ($0.1188$,
\texttt{ltr\_strict\_report.json}); the baseline script's group filtering
yields $0.122$, a filtering-level difference reported for transparency.}
\label{tab:ltr}
\end{table}

\paragraph{Result 1: no content-only model beats the fixed score.}
The query-blind probe, a LambdaMART free to combine the eight
non-degenerate sub-components nonlinearly, reaches $\rho = 0.104$
($0.088 \pm 0.029$ across seeds, the largest seed spread of any model)
with hit@1 $0.24$, against
$0.119$ and hit@1 $0.25$ for the fixed score, and ridge on the same
features reaches $0.075$. On unseen queries, no content-only model we can
fit beats the fixed, auditable score: it is not leaving recoverable
content-only signal on the table, and its modest outcome correlation
(Section~\ref{sec:outcome}) is best read as a lower bound sitting close
to the achievable ceiling of Proposition~\ref{prop:ceiling}, not as an
engineering shortfall.

\paragraph{Result 2: query conditioning recovers most of the rest, and
model-based relevance carries more of it than a leaky evaluation
suggests.} Conditioning on the query roughly triples the measured
signal, from $0.119$ to about $0.39$ (Figure~\ref{fig:signal}). The
strongest single signal is the open-weights cross-encoder logit at
$\rho = 0.388$, hit@1 $0.45$, with \emph{no fitting at all}; the fitted
models cluster at $0.36$ to $0.38$ and do not reliably beat it with 229
unique queries of training signal (full LambdaMART: $0.348 \pm 0.013$
across fold reshuffles). The zero-dependency regex-only ranker reaches
$0.222$ ($0.199 \pm 0.012$), hit@1 $0.34$: it roughly doubles the
content-only score and beats the position and BM25 baselines, but
recovers under two fifths of the measured gap
($0.10$ of the $0.119 \to 0.388$ span); the rest needs learned
representations. This corrects a
conclusion the record-fold evaluation would have licensed: under the
leaky folds, regex-only appeared to reach $0.53$ against $0.57$ for the
full stack, and we would have reported that regex features do most of
the work; under query-disjoint folds that attribution inverts
(Appendix~\ref{app:leakage}). Fold leakage manufactures qualitative
conclusions, not just inflated magnitudes.

% (Figure ref{fig:signal} appears in the body vignette.)

\paragraph{Robustness of the fit.} The LambdaMART configuration is the
harness default, fixed before evaluation; across the capacity grid
($100 \times 7$ to $400 \times 31$ trees $\times$ leaves) the strict-fold
value varies from $0.348$ to $0.368$, so configuration search could move
the headline by at most $0.02$, and the seed spread is $\pm 0.013$.
Feature importance of the regex-only model (gain shares, full-data fit;
Appendix~\ref{app:leakage}) confirms the mechanism: lexical query--source
overlap dominates (BM25 $10.5\%$, query-term frequency and Jaccard
$21\%$ combined), with source length ($10.6\%$) and two content features
(semantic\_coherence, information\_density) contributing secondary mass;
no single feature exceeds $11\%$.

\begin{figure}[t]
\centering
\includegraphics[width=0.95\linewidth]{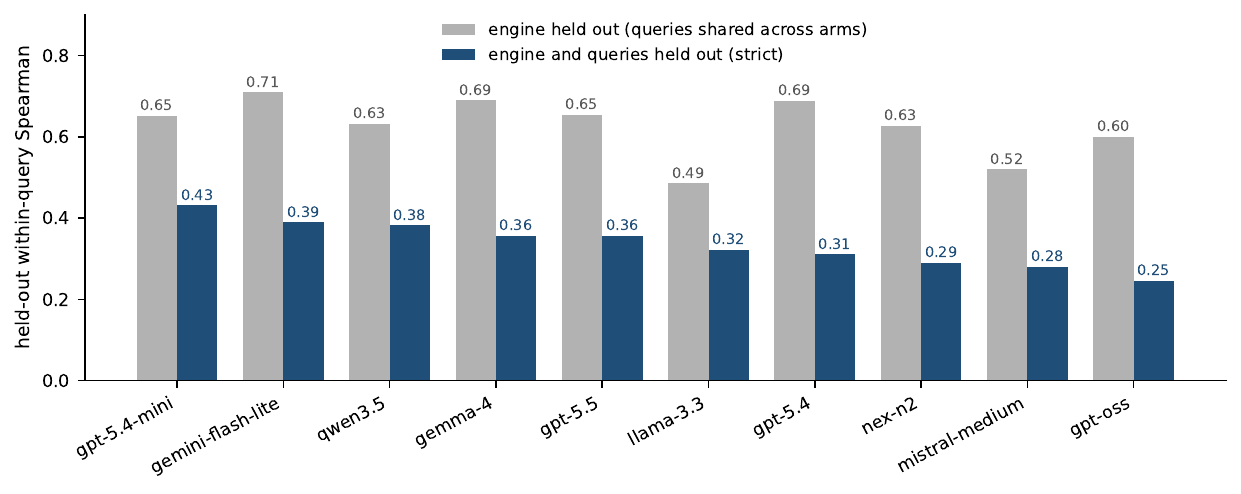}
\caption{Leave-one-engine-out transfer of the full LambdaMART, in two
regimes. Grey: only the engine is held out, so the training arms contain
the same queries under other engines and per-query outcome information
leaks across arms. Blue: the engine \emph{and} its evaluation queries are
held out (strict). The strict values are the defensible transfer claim;
the gap between the bars measures how much cross-engine query sharing
inflates the naive protocol.}
\label{fig:loeo}
\end{figure}

\paragraph{Generalization to unseen engines.} The transfer test is
leave-one-engine-out, and the fold lesson above applies here too, so we
report two regimes (Figure~\ref{fig:loeo}). Holding out only the engine
(training arms keep the same queries under other engines) gives a mean
held-out Spearman of $0.626$ over the ten arms; but since engines'
citation choices
correlate, this measures transfer of per-query outcome knowledge as much
as transfer of relevance patterns. The strict regime holds out the engine
\emph{and} its evaluation queries: the mean drops to $0.337$, ranging from
$0.25$ (gpt-oss) to $0.43$ (gpt-5.4-mini), with the four closed-weights
commercial arms all within that band (gemini-flash-lite $0.391$;
gpt-5.5 $0.357$, gpt-5.4 $0.311$, gpt-5.4-mini $0.432$, each trained
without any gpt-5.x arm's queries). The defensible claim is the strict
one: a ranker fit on the remaining engines, on disjoint queries, still
predicts a held-out commercial
engine's citations at two to three times the fixed score's per-engine
correlation, but far below what the naive protocol suggests. Both regimes
are released per engine (\path{ltr_strict_report.json}).

\paragraph{Which queries are predictable? (exploratory)} Grouping strict
held-out results by GEO-Bench query tags (membership, since tag positions
are inconsistent in the corpus) is unstable across dataset extensions,
which is itself the finding. On the original seven arms the data suggested
an archival-versus-volatile gradient (law and government $0.52$, health
$0.17$); under the ten-arm extension that gradient flattens (law and
government $0.35 \pm 0.04$, health $0.35 \pm 0.05$), and the
fact-versus-opinion contrast that looked significant under the leaky folds
($0.60$ versus $0.45$) remains noise ($0.36 \pm 0.02$ versus
$0.30 \pm 0.04$). The only tags that stand out at ten arms are
command-type queries, at chance ($-0.00 \pm 0.05$, $n = 76$), and
prediction, nominally lower ($0.26 \pm 0.06$). We report the tag analysis
as descriptive only; the full per-tag table is released
(\path{ltr_strict_report.json}).

\paragraph{Deployment implication.} The two artifacts serve different jobs
and should not be substituted for one another. The query-agnostic score is
the right tool where no query exists or none should: site audits, public
leaderboards, editorial quality control, regression testing of content
pipelines; its claims are the gate-enforced response surface, alignment
with historical causal evidence, and gaming resistance, not
per-query prediction, and Result~1 shows no content-only model does that
job better. The query-conditioned predictor is the right tool when
the question is ``for this query, which of my pages will the engine
cite?'': visibility forecasting, content targeting, and prioritizing pages
to improve for a given query portfolio; it needs a query at inference, and
under the strict protocol its most dependable form is the simplest, a
single open-weights cross-encoder over the candidates. Using the offline
score to promise per-query visibility overclaims
(Section~\ref{sec:outcome} quantifies by how much); using the ranker
where no query is available is undefined.

\section{The ranking-evaluation leakage audit}
\label{app:leakage}

The original ranking evaluation drew CV folds over (engine, query)
records; the corrected protocol draws them over unique queries
(Section~\ref{sec:ltr}). Table~\ref{tab:leak} tabulates both regimes for
every fitted model, from the released artifacts
(\path{ltr_leak_probe.json} preserves the record-fold values;
\path{ltr_strict_report_7eng.json} carries the corrected ones). The
audit ran on the original seven-arm dataset, where the leak was
discovered; the main text reports the ten-arm extension
(\path{ltr_strict_report.json}), whose strict values are within noise
of the corrected column here. Three
regularities are worth recording. First, inflation grows with model
flexibility: ridge moves little ($0.401 \to 0.379$), trees a lot, and the
query-blind probe most of all ($0.485 \to 0.065$), consistent with
memorization of per-query source outcomes rather than generic optimism.
Second, the leak reverses a qualitative conclusion: under record folds
the regex-only variant sits within $0.04$ of the full model and appears
to make model-based features dispensable; under query folds the gap is
$0.16$ and the conclusion inverts. Third, the no-fit baselines are
untouched by construction, which is what makes the cross-encoder-alone
row a fixed point for comparing the two regimes.

\begin{table}[t]
\centering
\footnotesize
\begin{tabular}{lcc}
\toprule
Model & record folds (leaky) $\rho$ & query-disjoint folds $\rho$ \\
\midrule
LambdaMART, content-only (query-blind) & $0.485$ & $0.065$ \\
LambdaMART, regex-only & $0.533$ & $0.202$ \\
LambdaMART, $+$ embedding & $0.553$ & $0.269$ \\
LambdaMART, $+$ cross-encoder & $0.571$ & $0.316$ \\
LambdaMART, full & $0.572$ & $0.363$ \\
GBM, content $+$ query & $0.447$ & $0.364$ \\
ridge, content $+$ query & $0.401$ & $0.379$ \\
cross-encoder logit alone (no fitting) & $0.393$ & $0.393$ \\
fixed score (no fitting) & $0.114$ & $0.114$ \\
\midrule
LOEO mean (engine-only vs.\ engine$+$query holdout) & $0.600$ & $0.318$ \\
\bottomrule
\end{tabular}
\caption{Every fitted model under both fold constructions, on the
original seven-arm dataset where the audit was performed. The
record-fold column is reported only as the comparison arm of this audit;
no claim in the paper rests on it.}
\label{tab:leak}
\end{table}

\begin{table}[t]
\centering
\footnotesize
\begin{tabular}{lc@{\hspace{2.2em}}lc}
\toprule
Feature & gain share & Feature & gain share \\
\midrule
q\_bm25 & $10.1\%$ & q\_tfidf\_cos & $5.5\%$ \\
c\_semantic\_coherence & $9.7\%$ & q\_lead\_overlap & $5.1\%$ \\
src\_len & $8.8\%$ & q\_phrase & $4.9\%$ \\
q\_tf\_sum & $8.1\%$ & q\_coverage & $4.5\%$ \\
q\_tf\_max & $7.4\%$ & q\_char\_dice & $4.1\%$ \\
q\_jaccard & $6.4\%$ & c\_information\_density & $6.3\%$ \\
\bottomrule
\end{tabular}
\caption{Feature importance (LightGBM gain shares, full-data fit,
interpretability only) of the regex-only ranker: the twelve largest of 25
features (\texttt{ltr\_strict\_report.json}). Lexical query--source
overlap dominates; two content features carry secondary mass; no feature
exceeds $10\%$.}
\label{tab:importance}
\end{table}

\section{Empirical feature diagnostics on GEO-Bench}
\label{app:diagnostics}

Table~\ref{tab:dist} summarizes the distribution of every sub-component on
the 500 unmodified GEO-Bench sources, and Table~\ref{tab:corr} their
Spearman correlations (script \path{geo_paper_feature_stats.py}). Three
facts documented here are load-bearing in the main text. (i)
\emph{Ceiling saturation}: shannon\_entropy is exactly 100 on 84\% of
sources and entity\_density on 57\%; base-corpus variance is not where
these features discriminate. (ii) \emph{Structural degeneracy}: mmr, ndcg,
and semantic\_redundancy take essentially two values (their degenerate
constant 20, or near 0 on parse failures), because plain-text GEO-Bench
passages rarely contain two valid markdown blocks; their pairwise
correlation of $1.0$ and their role as a parse-failure proxy follow
mechanically, which is why they are excluded from the ranking features of
Section~\ref{sec:ltr}. (iii) \emph{Redundancy of the entropy family}:
shannon\_entropy and information\_density correlate at $0.94$ (Pearson on
raw values; $0.50$ Spearman, reflecting the ceiling ties).

\begin{table}[t]
\centering
\footnotesize
\setlength{\tabcolsep}{4.5pt}
\begin{tabular}{lrrrrrrrr}
\toprule
Sub-component & mean & sd & p5 & p50 & p95 & \%$=100$ & \%$=0$ & \#unique \\
\midrule
shannon\_entropy & 99.3 & 6.4 & 97.9 & 100.0 & 100.0 & 84.0 & 0.4 & 36 \\
information\_density & 87.3 & 6.0 & 83.5 & 87.9 & 90.9 & 0.0 & 0.4 & 376 \\
quotable\_density & 48.9 & 43.3 & 0.0 & 55.6 & 100.0 & 28.6 & 40.4 & 24 \\
entity\_density & 86.6 & 21.2 & 42.6 & 100.0 & 100.0 & 57.4 & 0.8 & 164 \\
semantic\_coherence & 52.0 & 4.3 & 48.8 & 52.0 & 56.7 & 0.0 & 0.4 & 364 \\
self\_containment & 70.1 & 16.7 & 65.0 & 65.0 & 100.0 & 19.0 & 0.4 & 4 \\
statistic\_density & 17.4 & 29.5 & 0.0 & 0.0 & 100.0 & 6.6 & 60.6 & 84 \\
mmr\_score & 19.9 & 1.3 & 20.0 & 20.0 & 20.0 & 0.0 & 0.4 & 2 \\
citation\_f1 & 10.6 & 29.6 & 0.0 & 0.0 & 94.0 & 0.0 & 88.6 & 7 \\
ndcg\_score & 19.9 & 1.3 & 20.0 & 20.0 & 20.0 & 0.0 & 0.4 & 2 \\
semantic\_redundancy & 19.9 & 1.3 & 20.0 & 20.0 & 20.0 & 0.0 & 0.4 & 2 \\
\bottomrule
\end{tabular}
\caption{Distribution of the raw sub-components on the 500 unmodified
GEO-Bench sources (base variant of the released capture). The discrete
distributions of self\_containment, mmr, ndcg, and semantic\_redundancy
follow from the degenerate constants of Table~\ref{tab:degenerate} on
plain-text sources.}
\label{tab:dist}
\end{table}

\begin{table}[t]
\centering
\footnotesize
\setlength{\tabcolsep}{3.2pt}
\begin{tabular}{lrrrrrrrrrrr}
\toprule
 & 1 & 2 & 3 & 4 & 5 & 6 & 7 & 8 & 9 & 10 & 11 \\
\midrule
1 shannon\_entropy & 1.00 & 0.50 & 0.14 & $-$0.02 & $-$0.16 & 0.06 & 0.00 & 0.17 & 0.01 & 0.17 & 0.17 \\
2 information\_density & & 1.00 & 0.04 & $-$0.07 & $-$0.20 & 0.02 & $-$0.00 & 0.11 & $-$0.01 & 0.11 & 0.11 \\
3 quotable\_density & & & 1.00 & $-$0.07 & 0.12 & 0.01 & $-$0.02 & 0.07 & 0.04 & 0.07 & 0.07 \\
4 entity\_density & & & & 1.00 & $-$0.12 & $-$0.03 & 0.26 & 0.12 & 0.07 & 0.12 & 0.12 \\
5 semantic\_coherence & & & & & 1.00 & 0.02 & $-$0.06 & 0.11 & $-$0.08 & 0.11 & 0.11 \\
6 self\_containment & & & & & & 1.00 & 0.05 & 0.15 & 0.04 & 0.15 & 0.15 \\
7 statistic\_density & & & & & & & 1.00 & 0.05 & $-$0.06 & 0.05 & 0.05 \\
8 mmr\_score & & & & & & & & 1.00 & 0.02 & 1.00 & 1.00 \\
9 citation\_f1 & & & & & & & & & 1.00 & 0.02 & 0.02 \\
10 ndcg\_score & & & & & & & & & & 1.00 & 1.00 \\
11 semantic\_redundancy & & & & & & & & & & & 1.00 \\
\bottomrule
\end{tabular}
\caption{Spearman correlations between raw sub-components on the base
variant ($n = 500$). The structural trio (8, 10, 11) is perfectly mutually
correlated (two-valued); the entropy pair (1, 2) correlates at $0.50$
Spearman and $0.94$ Pearson. The anchor carriers (3, 7, 9) are nearly
orthogonal to everything, consistent with their sharp identification in
Table~\ref{tab:ident}.}
\label{tab:corr}
\end{table}

\end{document}